\documentclass[twoside]{article}
\usepackage[accepted]{aistats}

\usepackage{xspace}
\def\aggrevate{\textsc{AggreVaTe}\xspace}
\def\dagger{\textsc{DAgger}\xspace}

\usepackage{amsmath}
\usepackage{amsfonts}
\usepackage{amssymb}
\usepackage{amsthm}
\usepackage{bm}
\usepackage{bbm}
\usepackage{mathtools}
\usepackage{enumitem}
\usepackage{thmtools,thm-restate}
\usepackage{algorithm}
\usepackage{algorithmic}
\usepackage{natbib}
\usepackage{color}
\usepackage{graphicx}
\usepackage{comment}
\usepackage[latin1]{inputenc} 

\theoremstyle{plain}
\newtheorem{lemma}{Lemma}
\newtheorem{theorem}{Theorem}
\newtheorem{proposition}{Proposition}
\newtheorem{corollary}{Corollary}

\theoremstyle{definition}
\newtheorem{definition}{Definition}
\newtheorem{assumption}{Assumption}

\theoremstyle{remark}

\def\BB{\mathcal{B}}\def\CC{\mathcal{C}}

\def\HH{\mathcal{H}}

\def\NN{\mathcal{N}}

\def\WW{\mathcal{W}}\def\XX{\mathcal{X}}

\def\Abb{\mathbb{A}}
\def\Ebb{\mathbb{E}}

\def\Rbb{\mathbb{R}}
\def\Sbb{\mathbb{S}}

\def\R{\Rbb}

\def\const{\mathrm{const.}}

\newcommand{\norm}[1]{ \| #1  \|  }
\newcommand{\abs}[1]{ | #1 |  }

\newcommand{\lr}[2]{\langle #1, #2 \rangle}
\DeclareMathOperator*{\argmin}{arg\,min}

\newcommand{\E}{\Ebb}

\usepackage{etex,etoolbox}
\usepackage{environ}

\makeatletter
\providecommand{\@fourthoffour}[4]{#4}
\newcommand\fixstatement[2][\proofname\space of]{%
	\ifcsname thmt@original@#2\endcsname
	\AtEndEnvironment{#2}{%
		\xdef\pat@label{\expandafter\expandafter\expandafter
			\@fourthoffour\csname thmt@original@#2\endcsname\space\@currentlabel}%
		\xdef\pat@proofof{\@nameuse{pat@proofof@#2}}%
	}%
	\else
	\AtEndEnvironment{#2}{%
		\xdef\pat@label{\expandafter\expandafter\expandafter
			\@fourthoffour\csname #1\endcsname\space\@currentlabel}%
		\xdef\pat@proofof{\@nameuse{pat@proofof@#2}}%
	}%
	\fi
	\@namedef{pat@proofof@#2}{#1}%
}

\globtoksblk\prooftoks{1000}
\newcounter{proofcount}

\NewEnviron{proofatend}{%
	\edef\next{%
		\noexpand\begin{proof}[\pat@proofof\space\pat@label]%
			\unexpanded\expandafter{\BODY}}%
		\global\toks\numexpr\prooftoks+\value{proofcount}\relax=\expandafter{\next\end{proof}}
	\stepcounter{proofcount}}

\def\printproofs{%
	\count@=\z@
	\loop
	\the\toks\numexpr\prooftoks+\count@\relax
	\ifnum\count@<\value{proofcount}%
	\advance\count@\@ne
	\repeat}
\makeatother

\fixstatement{lemma}
\fixstatement{theorem}
\fixstatement{proposition}
\fixstatement{corollary}

\newif\ifLONG
\LONGfalse
\ifLONG 
\renewenvironment{proofatend}{\begin{proof}}{\end{proof}}
\fi

\begin{document}

\twocolumn[

\aistatstitle{Convergence of Value Aggregation 
	for Imitation Learning}
\aistatsauthor{ Ching-An Cheng \And Byron Boots }
\aistatsaddress{Gerogia Institute of Technology \And Georgia Institute of Technology}
]

\begin{abstract}
Value aggregation is a general framework for solving imitation learning problems. Based on the idea of data aggregation, it generates a policy sequence by iteratively interleaving policy optimization and evaluation in an online learning setting. While the existence of a good policy in the policy sequence can be guaranteed non-asymptotically, little is known about the convergence of the sequence or the performance of the last policy. In this paper, we debunk the common belief that value aggregation always produces a convergent policy sequence with improving performance. Moreover, we identify a critical stability condition for convergence and provide a tight non-asymptotic bound on the performance of the last policy. These new theoretical insights let us stabilize problems with regularization, which removes the inconvenient process of identifying the best policy in the policy sequence in stochastic problems.

\end{abstract}

\section{INTRODUCTION} \label{sec:introduction}
Reinforcement learning (RL) is a general framework for solving sequential decision problems~\citep{sutton1998introduction}. Using policy gradient methods, it has demonstrated impressive results in GO~\citep{silver2016mastering} and video-game playing~\citep{mnih2013playing}. 
However, due its generality, it can be difficult to learn a policy sample-efficiently or to characterize the performance of the found policy, which is critical in applications that involves real-world costs, such as robotics~\citep{pan2017agile}.  To better exploit the domain knowledge about a problem, one popular approach is imitation learning (IL)~\citep{pomerleau1989alvinn}. In this framework, instead of learning a policy from scratch, one leverages a black-box policy $\pi^*$, called the \emph{expert}, from which the learner can query demonstrations. The goal of IL is to identify a policy $\pi$ such that its performance is similar to or better than $\pi^*$. 

A recent approach to IL is based on the idea of data aggregation and online learning~\citep{ross2011reduction,sun2017deeply}.
The algorithm starts with an empty dataset and an initial policy $\pi_1$; in the $n$th iteration, the algorithm uses the current policy $\pi_n$ to gather new training data into the current dataset and then a supervised learning problem is solved on the updated dataset to compute the next policy $\pi_{n+1}$. By interleaving the optimization and the data collection processes in an online fashion, it can overcome the covariate shift problem in traditional batch IL~\citep{ross2011reduction}.

This family of algorithms can be realized under the general framework of value aggregation~\citep{ross2014reinforcement}, which has gained increasing attention due to its non-asymptotic performance guarantee. After $N$ iterations, a good policy $\pi$ \emph{exists} in the generated policy sequence $\{ \pi_n \}_{n=1}^N$ with performance  $J(\pi ) \leq J(\pi^*) + T \epsilon + \tilde{O}( \frac{1}{N}) $, where $J$ is the performance index, $\epsilon$  is the error due to the limited expressiveness of the policy class, and $T$ is the horizon of the problem. 
While this result seems strong at first glance, its guarantee concerns only the existence of a good policy and, therefore, is not ideal for stochastic problems. In other words, in order to find the best policy in $\{ \pi_n \}_{n=1}^N$ without incurring large statistical error, a sufficient amount of data must be acquired in each iteration, or all policies have to be memorized for a final evaluation with another large dataset~\citep{ross2011reduction}.

This inconvenience incentivizes practitioners to just return the last policy $\pi_N$~\citep{laskey2017comparing}, and, anecdotally, the last policy $\pi_N$ has been reported to have good empirical performance~\citep{ross2013learning,pan2017agile}. 
Supporting this heuristic is the insight that the last policy $\pi_N$ is trained with \emph{all} observations and therefore \emph{ideally} should perform the best. Indeed, such idealism works when all the data are sampled i.i.d., as in the traditional batch learning problems~\citep{vapnik1998statistical}. However, because here new data are collected using the updated policy in each iteration, whether such belief applies depends on the convergence of the distributions generated by the policy sequence. 

While \cite{ross2014reinforcement} alluded that ``\dots the distribution of visited states converges over the iterations of learning.'',  we show this is \emph{not} always true---the convergence is rather problem-dependent. In this paper, we identify a critical stability constant $\theta$ that determines the convergence of the  policy sequence. We show that there is a simple example (in Section~\ref{sec:convergence?}) in which the policy sequence diverges when $\theta > 1$. In Section~\ref{sec:main results}, we provide tight non-asymptotic bounds on the performance of the last policy $\pi_N$, in both  deterministic and stochastic problems, which implies that the policy sequence always converges when $\theta < 1$. 
Our new insight also suggests that the stability of the last policy $\pi_N$ can be recovered by regularization, as discussed in Section~\ref{sec:regularization}. 

\section{PROBLEM SETUP} \label{sec:problem setup}
We consider solving a discrete-time RL problem. Let $\Sbb$ be the state space and $\Abb$ be the action space of an agent. Let $\Pi$ be the class of policies and let $T$ be the length of the planning horizon.\footnote{A similar analysis can be applied to discounted infinite-horizon problems.} The objective of the agent is to search for a policy $\pi\in\Pi$ to minimize an accumulated cost $J(\pi)$: 
\begin{align}
	\min_{\pi \in \Pi} J(\pi) \coloneqq \min_{\pi \in \Pi}  \E_{\rho_{\pi}} \left[ \sum_{t=0}^{T-1} c_t(s_t, a_t) \right]  \label{eq:RL problem}
\end{align}
in which $c_t$ is the instantaneous cost at time $t$, and $\rho_\pi$ denotes the trajectory distribution of $(s_t, a_t) \in \Sbb \times \Abb$, for $t = 1, \dots, T$, under policy $a_t \sim \pi(s_t)$ given a prior distribution $p_0(s_0)$. Note that we do not place assumptions on the structure of $\Sbb$ and $\Abb$ and the policy class $\Pi$. To simplify the notation, we write $\E_{a \sim \pi} $ even if the policy is deterministic.

For notation:  we denote $Q_{\pi|t}(s,a)$ as the Q-function at time $t$ under policy $\pi$ and $V_{\pi|t}(s) = \E_{a \sim \pi}[Q_{\pi|t}(s,a)]$ as the associated value function. In addition, we introduce some shorthand: we denote $d_{\pi|t}(s)$ as the state distribution at time $t$ generated by running the policy $\pi$ for the first $t$ steps, and define a joint distribution $d_{\pi}(s,t) = d_{\pi|t}(s)U(t) $, where $U(t)$ is the uniform distribution over the set $\{0,\dots, T-1\}$. Due to space limitations, we will often omit explicit dependencies on random variables in expectations, e.g. we will write
\begin{align}
\min_{\pi \in \Pi} \E_{d_{\pi}} \E_{ \pi} \left[c_t \right]  \label{eq:simple form of RL problem}
\end{align}
to denote $ \min_{\pi \in \Pi} \E_{s,t \sim d_{\pi}} \E_{a \sim \pi} \left[c_t(s, a) \right] $, which is equivalent to $\min_{\pi \in \Pi} \frac{1}{T} J(\pi)$ (by definition of $d_{\pi}$).

\section{VALUE AGGREGATION} \label{sec:value aggregation}
Solving general RL problems is challenging. In this paper, we focus on a particular scenario, in which the agent, or the \emph{learner}, has access to an \emph{expert} policy $\pi^*$ from which the learner can query demonstrations. Here we embrace a general notion of expert. While it is often preferred that the expert is nearly optimal in~\eqref{eq:RL problem}, the expert here can be \emph{any} policy, e.g. the agent's initial policy. Note, additionally, that the RL problem considered here is not necessarily directly related to a real-world application; it can be a surrogate problem which arises in solving the true problem.

The goal of IL is to find a policy $\pi$ that outperforms or behaves similarly to the expert $\pi^*$ in the sense that $J(\pi) \leq J(\pi^*) + O(T)$. That is, we treat IL as performing a robust, approximate policy iteration step from $\pi^*$: ideally IL should lead to a policy that outperforms the expert, but it at least returns a policy that performs similarly to the expert.

\aggrevate (Aggregate Value to Imitate) is an IL algorithm proposed by \cite{ross2014reinforcement}  based on the idea of online convex optimization~\citep{hazan2016introduction}. 
Here we give a compact derivation and discuss its important features in preparation for the analysis in Section~\ref{sec:main results}. To this end, we introduce the performance difference lemma due to~\citet{kakade2002approximately}, which will be used as the foundation to derive \aggrevate.
\begin{lemma} \label{lm:performance difference}
	\citep{kakade2002approximately} Let $\pi$ and $\pi'$ be two policies and  $ A_{\pi'|t}(s, a) = Q_{\pi'|t}(s,a) - V_{\pi'|t}(s)$ be the (dis)advantage function at time $t$ with respect to running $\pi'$. Then it holds that
	\begin{align}
	J(\pi) = J(\pi') +  T \E_{s,t \sim d_\pi} \E_{a \sim \pi} [ A_{\pi'|t}(s,a)  ].
	\end{align}
\end{lemma}

\subsection{Motivation}
The main idea of \aggrevate is to minimize the performance difference between the learner's policy and the expert policy, which, by Lemma~\ref{lm:performance difference}, is given as $\frac{1}{T}\left(J(\pi) - J(\pi^*)\right) = \E_{ d_\pi} \E_{\pi} [ A_{\pi^*|t}(s,a)  ]$. \aggrevate can be viewed as solving an RL problem  with  $A_{\pi^*|t}(s,a)$ as the instantaneous cost at time $t$:
\begin{align}
\min_{\pi \in \Pi} \E_{d_{\pi}} \E_{ \pi} \left[ A_{\pi^*|t}  \right]. \label{eq:equivalent RL}
\end{align}

Although the transformation from~\eqref{eq:simple form of RL problem} to \eqref{eq:equivalent RL} seems trivial, it unveils some critical properties. Most importantly, the range of the problem in~\eqref{eq:equivalent RL} is normalized. For example, regardless of the original definition of $c_t$, if $\Pi \ni \pi^*$, there exists at least a policy $\pi \in \Pi$ such that \eqref{eq:equivalent RL} is non-positive (i.e. $J(\pi) \leq J(\pi^*)$). As now the problem~\eqref{eq:equivalent RL} is relative, it becomes possible to place a qualitative assumption to bound the performance in~\eqref{eq:equivalent RL} in terms of some measure of expressiveness of the policy class $\Pi$.

We formalize this idea into Assumption~\ref{as:good approximator}, which is one of the core assumptions implicitly imposed by \cite{ross2014reinforcement}.\footnote{The assumption is implicitly made when \cite{ross2014reinforcement} assume the existence of $\epsilon_{\text{class}}$ in Theorem 2.1 on page 4.}
To simplify the notation, we define a function $F$ such that for any two policies $\pi, \pi'$
\begin{align}
F(\pi', \pi) \coloneqq \E_{d_{\pi'}} \E_{ \pi} \left[ A_{\pi^*|t}  \right] 
\label{eq:def of F(x,x)}
\end{align}
This function captures the main structure in~\eqref{eq:equivalent RL}. By separating the roles of $\pi'$ (which controls the state distribution) and $\pi$ (which controls the reaction/prediction), the performance of a policy class $\Pi$ relative to an expert $\pi^*$ can be characterized with the approximation error in a supervised learning problem. 

\begin{assumption} \label{as:good approximator}
Given a policy $\pi^*$, the policy class $\Pi$ satisfies that for arbitrary sequence of policies $\{ \pi_n \in \Pi  \}_{n=1}^N$, there exists a small constant $\epsilon_{\Pi,\pi^*}$  such that 
\begin{align}
	\min_{\pi \in \Pi} \frac{1}{N}f_{1:N}(\pi )\leq \epsilon_{\Pi,\pi^*},
\end{align}
where $f_n(\pi) \coloneqq F(\pi_n, \pi)$ and $f_{1:n}(\pi) = \sum_{n=1}^N f_n(\pi)$. 
\end{assumption}
This assumption says that there exists at least a policy $\pi \in \Pi$ which is as good as $\pi^*$ in the sense that $\pi$ can predict $\pi^*$ well in a cost-sensitive supervised learning problem, with small error $\epsilon_{\Pi,\pi^*}$, 
under the average state distribution generated by an \emph{arbitrary} policy sequence $\{ \pi_n \in \Pi  \}_{n=1}^N$.

Following this assumption, \aggrevate exploits another critical structural property of the problem. 
\begin{assumption} \label{as:strong convexity}
$\forall \pi' \in \Pi$, $F(\pi', \pi)$ is a strongly convex function in $\pi$. 
\end{assumption}
While \citet{ross2014reinforcement} did not explicitly discuss under which condition Assumption~\ref{as:strong convexity} holds, here we point out some examples (proved in Appendix~\ref{sec:proofs}). 

\begin{proposition} \label{pr:convex examples}
Suppose  $\Pi$ consists of deterministic linear policies (i.e. $a = \phi(s)^T x$ for some feature map $\phi(s)$ and weight $x$) and  $\forall s\in \Sbb$,  $c_t(s,\cdot)$ is strongly convex. Assumption~\ref{as:strong convexity} holds under any of the following:
\begin{enumerate}
\item $V_{\pi^*|t}(s)$ is constant over $\Sbb$ (in this case $A_{\pi^*|t}(s,a)$ is equivalent to $ c_t(s,a) $  up to a constant  in $a$)
\item  The problem is continuous-time and the dynamics are affine in action. 
\end{enumerate}
\end{proposition}
\begin{proofatend}
Let $\pi$ be parametrized by $x$. We prove the sufficient conditions by showing that $A_{\pi^*|t}(s,a)$ is strongly convex in $a$ for all $s \in \Sbb$, which by the linear policy assumption implies $f_n(\pi)$ is strongly convex in $x$.

For the first case, since $ Q_{\pi^*|t}(s,a) = c_t(s,a) + \E_{s'|s,a} [V_{\pi^*|t+1}(s')] $, given the constant assumption, it follows that 
\[ A_{\pi^*|t}(s,a)= Q_{\pi^*|t}(s,a) - V_{\pi^*|t}(s)  = c_t(s,a) + \const\]
is strongly convex  in terms of $a$.

For the second case, consider a system 
 $ds = \left( f(s)+g(s)a \right) dt + h(s) dw$, where $f,g,h$ are some matrix functions and $dw$ is a Wiener process. By Hamilton-Jacobi-Bellman equation~\citep{bertsekas1995dynamic}, the advantage function can be written as
\begin{align*}
A_{\pi^*|t}(s,a) = c_t(s,a) + \partial_s V_{\pi^*|t}(s)^T  g(s) a  + r(s)
\end{align*}
where $r(s)$ is some function in $s$. Therefore, $ A_{\pi^*|t}(s,a) $ is strongly convex in $a$.
\end{proofatend}

We further note that \aggrevate has demonstrated impressive empirical success even when Assumption~\ref{as:strong convexity} cannot be verified~\citep{sun2017deeply,pan2017agile}.

\subsection{Algorithm and Performance}
Given Assumption~\ref{as:strong convexity}, \aggrevate treats  $f_n(\cdot)$ as the per-round cost in an online convex optimization problem and updates the policy sequence as follows: Let $\pi_1$ be an initial policy. In the $n$th iteration of \text{AggreVaTe}, the policy is updated by\footnote{We adopt a different notation from \cite{ross2014reinforcement}, in which the per-round cost $\E_{d_{\pi_n}} \E_{ \pi} \left[ Q_{\pi^*|t}  \right]$ was used. Note these two terms are equivalent up to an additive constant, as the optimization here is over $\pi$ with $\pi_n$ fixed. }
\begin{align}
\pi_{n+1} = \argmin_{\pi \in \Pi} f_{1:n}(\pi). \label{eq:AggreVate udpate}
\end{align}
After $N$ iterations, the best policy in the sequence $\{ \pi_n\}_{n=1}^N$ is returned, i.e. $\pi = \hat{\pi}_N$, where
\begin{align}
\hat{\pi}_N \coloneqq \argmin_{\pi \in \{ \pi_n\}_{n=1}^N}  J(\pi). \label{eq:AggreVate search}
\end{align}
As the update rule~\eqref{eq:AggreVate udpate} (aka Follow-the-Leader) has a sublinear regret, it can be shown that (cf. Section~\ref{sec:classical results}) 
\begin{align}
J(\hat{\pi}_N) \leq J(\pi^*) + T\left( \epsilon_{\text{class}} + \epsilon_{\text{regret}} \right),  \label{eq:original AggreVaTe guarantee}
\end{align}
in which  $\epsilon_{\text{regret}}= \tilde{O}(\frac{1}{N}) $ is the average regret and 
\begin{align*}
\epsilon_{\text{class}} \coloneqq \min_{\pi \in \Pi}  \frac{1}{N}  \sum_{n=1}^{N} \E_{d_{\pi_n}} \left[ \E_{ \pi}[Q_{\pi^*|t}] - \E_{ \pi^* }[Q_{\pi^*|t}] \right]
\end{align*}
compares the best policy in the policy class $\Pi$ and the expert policy $\pi^*$. The term $\epsilon_{\text{class}}$ can be negative if there exists a policy in $\Pi$ that is better than  $\pi^*$ under the average state distribution, $\frac{1}{N} \sum_{n=1}^{N}d_{\pi_n}$, generated by \aggrevate. By Assumption~\ref{as:good approximator},  $\epsilon_{\text{class}}\leq \epsilon_{\Pi,\pi^*}$; we know $\epsilon_{\text{class}}$ at least should be small.

The performance bound in \eqref{eq:original AggreVaTe guarantee} satisfies the requirement of IL that $J(\hat{\pi}_N) \leq J(\pi^*) + O(T)$. Especially because $\epsilon_{\text{class}}$ can be non-positive, \aggrevate can be viewed as robustly performing one approximate policy iteration step from $\pi^*$. 

One notable special case of \aggrevate is \dagger \citep{ross2011reduction}. \dagger tackles the problem of solving an unknown RL problem by imitating a desired policy $\pi^*$. The reduction to \aggrevate can be seen by setting $c_t(s,a) = \E_{a^* \sim \pi^* } [\norm{a - a_t^*}]$ in \eqref{eq:RL problem}. In this case, $\pi^*$ is optimal for this specific choice of cost and therefore $V_{\pi^*|t}(s) = 0$. By Proposition~\ref{pr:convex examples}, $A_{\pi^*|t}(s,a)=c_t(s,a)$  and  $\epsilon_{\text{class}}$ reduces to $\min_{\pi \in \Pi}  \frac{1}{N}  \sum_{n=1}^{N} \E_{d_{\pi_n}}  \E_{ \pi}[c_t]  \geq 0 $, which is related to the expressiveness of the policy class $\Pi$.

\section{GUARANTEE ON THE LAST POLICY?} \label{sec:convergence?}

The performance bound in Section~\ref{sec:value aggregation} implicitly assumes that the problem is either deterministic or that infinite samples are available in each iteration.
For stochastic problems, $f_{1:n}$ can be approximated by finite samples or by function approximators~\citep{ross2014reinforcement}. 
Suppose $m$ samples are collected in each iteration to approximate $f_n$. An additional error in $O(\frac{1}{\sqrt{mN}})$ will be added to the performance of $\hat{\pi}_N$. However, in practice, another constant statistical error\footnote{The original analysis in the stochastic case by \cite{ross2014reinforcement} only guarantees the existence of a good policy in the sequence. The $O(\frac{1}{m})$ error is due to identifying the best policy~\citep{lee1998importance} (as the function is strongly convex) and the $O(\frac{1}{\sqrt{mN}})$ error is the generalization error~\citep{cesa2004generalization}.} in  $O(\frac{1}{m})$ is introduced when one attempts to identify $\hat{\pi}_N$ from the sequence $\{\pi_n\}_{n=1}^N$.

This practical issue motivates us to ask whether a similar guarantee applies to the last policy $\pi_N$ so that the selection process to find $\hat{\pi}_N$ can be removed. In fact, the last policy $\pi_n$ has been reported to have good performance empirically~\citep{ross2013learning,pan2017agile}. It becomes interesting to know what one can say about $\pi_N$. It turns out that running \aggrevate does not always yield a policy sequence $\{ \pi_n \}$ with reasonable performance, as given in the  example below.

\paragraph{A Motivating Example} Consider  a two-stage deterministic optimal control problem: 
\begin{align}
\min_{\pi \in \Pi} J(\pi) = \min_{\pi \in \Pi} c_1(s_1, a_1) + c_2(s_2, a_2) \label{eq:online-DA counter example}
\end{align}
where the  transition and costs are given as 
\begin{align*}
s_1 = 0, &\quad s_2 = \theta (s_1 + a_1), \\
 c_1(s_1, a_1) = 0, &\quad c_2(s_2, a_2) = (s_2 - a_2)^2.
\end{align*}
Since the problem is deterministic, we consider a policy class $\Pi$ consisting of open-loop stationary deterministic policies, i.e. $a_1 = a_2 = x$ for some $x$ (for convenience $\pi$ and $x$ will be used interchangeably). It can be easily seen that $\Pi$ contains a globally optimal policy, namely $x=0$.  We perform \aggrevate with a feedback expert policy $a^*_t= s_t$ and some initial policy $|x_1|>0$. While it is a custom to initialize $x_1 = \argmin_{x \in \XX} F(x^*,x)$ (which in this case would ideally return $x_1 =0$), setting $|x_1|>0$ simulates the effect of finite numerical precision. 

We consider two cases ($\theta > 1$ or $\theta < 1$) to understand the behavior of \aggrevate. First, suppose $\theta > 1$. Without loss generality, take $\theta = 10$ and $x_1 =1$. We can see running \aggrevate will generate a divergent sequence $x_2 = 10, x_3 = 55,  x_4 = 220 \dots$ (in this case \aggrevate would return $x_1$ as the best policy).  Since $J(x) = (\theta-1)^2 x^2$, the performance $\{J(x_n)\}$ is an increasing sequence. Therefore, we see even in this simple case, which can be trivially solved by gradient descent in $O(\frac{1}{n})$, using \aggrevate results in a sequence of policies with  degrading performance, though the policy class $\Pi$ includes a globally optimal policy.  Now suppose on the contrary $\theta < 1$. We can see that $\{x_n\}$ asymptotically converges to $x^*=0$.

This example illustrates several important properties of \aggrevate.
It illustrates that whether \aggrevate can generate a reasonable policy sequence or not depends on intrinsic properties of the problem (i.e. the value of $\theta$). The non-monotonic property was also empirically found in \citet{laskey2017comparing}.
In addition,  it shows that $\epsilon_{\Pi,\pi^*}$ can be large while $\Pi$ contains an optimal policy.\footnote{In this example, $\epsilon_{\Pi,\pi^*}$ can be arbitrarily large unless $\XX$ is bounded. However, even when $\epsilon_{\Pi,\pi^*}$ is bounded, the performance of the policy sequence can be non-monotonic.}
This suggests that Assumption~\ref{as:good approximator} may be too strong, especially in the case where $\Pi$ does not contain $\pi^*$.

\section{THEORETICAL ANALYSIS} \label{sec:main results}

Motivated by the example in Section~\ref{sec:convergence?}, we investigate the convergence of the policy sequence generated by \aggrevate in general problems. We assume the policy class $\Pi$ consists of policies parametrized by some parameter $x \in \XX$, in which $\XX$ is a convex set in a normed space with norm $\norm{\cdot}$ (and $\norm{\cdot}_*$ as its dual norm). With abuse of notation, we abstract the RL problem in \eqref{eq:equivalent RL} as 
\begin{align} \label{eq:structured optimization}
\min_{x\in \XX} F(x,x)
\end{align}
where we overload the notation $F(\pi',\pi)$ defined in \eqref{eq:def of F(x,x)} as $F(\pi',\pi) = F(y,x)$ when $\pi,\pi' \in \Pi$ are parametrized by $x, y \in \XX$, respectively. Similarly, we will write  $f_n(x) = F(x_n, x)$ for short. In this new notation,  \aggrevate's update rule in~\eqref{eq:AggreVate udpate} can be simply written as $x_{n+1} = \argmin_{x \in \XX} f_{1:n}(x)$. 

Here we will focus on the bound on $F(x,x)$, because, for $\pi$  parameterized by $x$, this result can be directly translated to a bound on $J(\pi)$: by definition of $F$ in~\eqref{eq:def of F(x,x)} and Lemma~\ref{lm:performance difference}, $J(\pi) = J(\pi^*) + T F(\pi, \pi) $.  For simplicity, we will assume for now  $F$ is deterministic; the convergence in stochastic problems will be discussed at the end of the section.

\subsection{Classical Result}\label{sec:classical results}

For completeness, we restate the structural assumptions made by \aggrevate in terms of $\XX$ and  review the known convergence of \aggrevate \citep{ross2014reinforcement}.

\begin{assumption} \label{as:structured function}
Let  $\nabla_2$ denote the derivative with respect to the second argument. 
\begin{enumerate} [ref={\theassumption.\arabic*}]
\item \label{as:uniform strongly convex2} $F$ is uniformly $\alpha$-strongly convex in the second argument:  $\forall x,y,z  \in \XX$, 
$ F(z,x) \geq F(z,y) + \lr{\nabla_2 F(z,y) }{x-y} + \frac{\alpha}{2} \norm{x-y}^2 $.
\item \label{as:uniform Lipchitz2} 
$F$ is uniformly $G_2$-Lipschitz continuous  in the second argument: $\forall x,y,z  \in \XX$, $ \abs{F(z,x) - F(z,y)} \leq G_2 \norm{x-y} $	. 
\end{enumerate}

\end{assumption}
\begin{assumption} \label{as:good approximator in X}
	 $\forall \{ x_n \in \XX  \}_{n=1}^N$, there exists a small constant $\epsilon_{\Pi,\pi^*}$ such that 
$
	\min_{x \in \XX} \frac{1}{N}f_{1:N}(x )\leq \epsilon_{\Pi,\pi^*}.
$
\end{assumption}

\begin{theorem} \label{th:aggrevate classical result}
	Under Assumption~\ref{as:structured function} and \ref{as:good approximator in X}, \aggrevate generates a sequence such that, for all $N \geq 1$,
	\begin{align*}
	F(\hat{x}_N, \hat{x}_N) \leq\frac{1}{N} \sum_{n=1}^{N}  f_n( x_n) \leq \epsilon_{\Pi,\pi^*} + \frac{G_2^2}{2\alpha}\frac{\ln(N) + 1 }{N}
	\end{align*}
	where
	$\hat{x}_N \coloneqq \argmin_{x \in \{ x_n\}_{n=1}^N}  F(x_n, x_n)$. 
\end{theorem}

\begin{proofatend}
The proof is based on a basic perturbation lemma in convex analysis (Lemma~\ref{lm:perturbationLemma}), which for example can be found in~\citep{mcmahan2014survey}, and a lemma for online learning (Lemma~\ref{lm:seqcostLemma}). 
\begin{lemma} \label{lm:perturbationLemma}
	Let $\phi_1: \R^d \mapsto \R \bigcup \{ \infty \}$ be a convex function such that $x_1 = \arg \min_x \phi_t(x) $ exits. Let $\psi$ be a  function such that $\phi_2(x) = \phi_1(x) + \psi(x)$ is $\alpha$-strongly convex with respect to $\norm{\cdot}$. Let $x_2 = \arg \min_x \phi_2(x)$. Then, for any $g \in \partial \psi(x_1)$, we have 
	\[
	\norm{x_1 - x_2} \leq \frac{1}{\alpha} \norm{g}_{*}
	\]
	and for any $x'$
	\[
	\phi_2(x_1) - \phi_2(x') \leq \frac{1}{2\alpha} \norm{g}_*^2
	\]
	When $\phi_1$ and $\psi$ are quadratics (with $\psi$ possibly linear) the above holds with equality.
\end{lemma}
\begin{lemma} \label{lm:seqcostLemma}
	Let $l_t(x)$ be a sequence of  functions. Denote $l_{1:t}(x) = \sum_{\tau=1}^{t} l_\tau(x).$ and let 
	\[
	x_t^* = \arg \min_{x\in K} l_{1:t}(x)
	\] Then for any sequence $\{x_1, \dots, x_T\}$, $\tau \geq 1$, and any $ x^* \in K$, it holds
	\begin{align*}
	\sum_{t=\tau}^{T} l_t( x_t)  
	&\leq   l_{1:T}(x_T^*) -  l_{1:\tau-1}(x_{\tau-1}^*)  \\
	&\quad +  \sum_{t=\tau}^{T} l_{1:t}(x_t)- l_{1:t}( x_{t}^*) 	
	\end{align*}
\end{lemma}
\begin{proof} Introduce a slack loss function $l_0(\cdot) = 0$ and define $x_0^* = 0$ for index convenience. This does not change the optimum, since $l_{0:t}(x) =  l_{1:t}(x)$.
	\begin{align*}
	\sum_{t=\tau}^{T} l_t( x_t)
	&= \sum_{t=\tau}^{T} l_{0:t}( x_t)- l_{0:t-1}( x_t) \\
	&\leq \sum_{t=\tau}^{T} l_{0:t}(x_t)- l_{0:t-1}( x_{t-1}^*)   \\
	&= l_{0:T}(x_T^*) -  l_{0:\tau-1}(x_{\tau-1}^*) \\
	&\quad +  \sum_{t=\tau}^{T} l_{0:t}(x_t)- l_{0:t}( x_{t}^*)    \qedhere
	\end{align*}
\end{proof}	
Note Lemma~\ref{lm:seqcostLemma} does not require $l_t$ to be convex and the minimum to be unique.

To prove Theorem~\ref{th:aggrevate classical result}, we first note that by definition of $\hat{x}_N$, it satisfies
$
	F(\hat{x}_N, \hat{x}_N) \leq \frac{1}{N} \sum_{n=1}^{N} f_n(x_n)
$. To bound the average performance,  we 
use Lemma~\ref{lm:seqcostLemma} and write
	\begin{align*}
	\sum_{n=1}^{N}  f_n( x_n)   \leq  f_{1:N}(x_{N+1})  + \sum_{n=1}^{N} f_{1:n}( x_n) - f_{1:n}(x_{n+1})
	\end{align*}
	since $x_n = \argmin_{x \in \XX} f_{1:n-1}(x)$. Then because $f_{1:k}$ is $k\alpha$-strongly convex, by Lemma~\ref{lm:perturbationLemma},
	\begin{align*}
	\sum_{n=1}^{N}  f_n( x_n)   \leq  f_{1:N}(x_n^*)  + \sum_{n=1}^{N} \frac{ \norm{\nabla f_n( x_n)}_*^2 }{ 2  \alpha n} .
	\end{align*}
	Finally, dividing the upper-bound by $n$ and using the facts that $\sum_{k=1}^{n}\frac{1}{k} \leq \ln(n)+1$ and $\min a_i \leq \frac{1}{n} \sum a_i$ for any scalar sequence $\{ a_n \}$, we have the desired result.
\end{proofatend}
\ifLONG \else
\begin{proof}
Here we present a sketch (see Appendix~\ref{sec:proofs} for details). The first inequality is straightforward. To bound the average performance, it can be shown that 
$
\sum_{n=1}^{N}  f_n( x_n)   \leq  \min_{x \in \XX} f_{1:N}(x)  + \sum_{n=1}^{N} f_{1:n}( x_n) - f_{1:n}(x_{n+1})
$. 
Since $x_n$ minimizes $f_{1:n-1}$ and $f_{1:n}$ is $n\alpha$-strongly convex  ,   $f_{1:n}(x_n)$ is  upper  bounded by $ f_{1:n-1}(x_n) + \frac{ \norm{\nabla f_n( x_n)}_*^2 }{ 2  \alpha n} $, where $\norm{\nabla f_n( x_n)}_* \leq G_2$. This concludes the proof.
\end{proof}
\fi

\subsection{New Structural Assumptions}

\aggrevate can be viewed as an attempt to solve the optimization problem in~\eqref{eq:structured optimization} without any information (not even continuity) regarding how $F(x,x)$ changes with perturbations in the first argument. Since making even a local improvement for general Lipschitz continuous problems is known to be NP-hard~\citep{nesterov2013introductory}, the classical performance guarantee of \aggrevate is made possible, only because of the additional structure given in Assumption~\ref{as:good approximator in X}. However, as discussed in Section~\ref{sec:convergence?}, Assumption~\ref{as:good approximator in X} can be too strong and is yet insufficient to determine if the performance of the last policy can improve over iterations. Therefore, to analyze the performance of the last policy, we require additional structure on $F$.

Here we introduce a continuity assumption.
\begin{assumption} \mbox{}
	 \label{as:uniform continous gradient2} $\nabla_2 F$ is uniformly $\beta$-Lipschitz continuous  in the first argument: $\forall x,y,z  \in \XX$ 
	$
	\norm{ \nabla_2 F(x,z) - \nabla_2 F(y,z)}_* \leq \beta \norm{x-y}
	$.
\end{assumption}
Because the first argument of $F$ in~\eqref{eq:def of F(x,x)} defines the change of state distribution, Assumption~\ref{as:uniform continous gradient2} basically requires that the expectation over $d_{\pi}$ changes continuously with respect to $\pi$, which is satisfied in most RL problems. 
Intuitively, this quantifies the difficulty of a problem in terms of how sensitive the state distribution is to policy changes.

In addition, we relax Assumption~\ref{as:good approximator in X}. As shown in Section~\ref{sec:convergence?}, Assumption~\ref{as:good approximator in X} is sometimes too strong, because it might not be satisfied even when $\Pi$ contains a globally optimal policy. In the analysis of convergence, we instead rely on a necessary condition of Assumption~\ref{as:good approximator in X}, which is satisfied by the example in Section~\ref{sec:convergence?}. 
\begin{assumption} \label{as:relaxed good approximator in X}
	Let $\pi$ be a policy parametrized by $x$. There exists a small constant $\tilde{\epsilon}_{\pi,\pi^*}$ such that $\forall x \in \XX$, 
	$
	\min_{y \in \XX} F(x,y) \leq \tilde{\epsilon}_{\Pi,\pi^*}
	$. 
\end{assumption}
Compared with the global Assumption~\ref{as:good approximator in X}, the relaxed condition here is only \emph{local}: it only requires the existence of a good policy with respect to the state distribution visited by running a \emph{single} policy. It can be easily shown that  $\tilde{\epsilon}_{\Pi,\pi^*} \leq \epsilon_{\Pi,\pi^*}$. 

\subsection{Guarantee on the Last Policy}

In our analysis, we define a stability constant $\theta = \frac{\beta}{\alpha}$. One can verify that this definition agrees with the $\theta$ used  in the example in Section~\ref{sec:convergence?}. This stability constant will play a crucial role in determining the convergence of $\{ x_n \}$, similar to the spectral norm of the Jacobian matrix in discrete-time dynamical systems~\citep{antsaklis2007linear}.
We have already shown above that  if $\theta > 1$ there is a problem such that \aggrevate generates a divergent sequence $\{ x_n \}$ with degrading performance over iterations. We now show that if $\theta < 1$, then $\lim_{n\to \infty }F(x_n, x_n ) \leq \tilde{\epsilon}_{\Pi, \pi^*}$ and moreover $\{x_n\}$ is convergent.

\begin{theorem} \label{th:convergence of the last iterate}
	Suppose Assumptions~\ref{as:structured function}, \ref{as:uniform continous gradient2}, and~\ref{as:relaxed good approximator in X} are satisfied. Let $\theta = \frac{\beta}{\alpha}$. Then for all $N\geq 1$ it holds
	\begin{align*}
	F(x_N, x_N) \leq  \tilde{\epsilon}_{\Pi,\pi^*} +   \frac{\left( \theta e^{1-\theta} G_2 \right)^2}{2\alpha}  N^{2(\theta -1)}
	\end{align*}
	and  $\norm{ x_N- \bar{x}_N} = \frac{G_2 e^{1-\theta} }{\alpha} N^{\theta -1}$, where $\bar{x}_N = \frac{1}{N} x_{1:N}$. 
	In particular, if $\theta < 1$, then $\{ x_n \}_{n=1}^{\infty}$ is convergent
\end{theorem}
Theorem~\ref{th:convergence of the last iterate} implies that the stability and convergence of \aggrevate depends solely on the problem properties. If the state distribution $d_{\pi}$ is sensitive to minor policy changes, running \aggrevate would fail to provide any guarantee on the last policy. Moreover, Theorem~\ref{th:convergence of the last iterate} also characterizes the performance of the average policy $\bar{x}_N$ when $\theta < 1$, .

The upper bound in Theorem~\ref{th:convergence of the last iterate} is tight, as indicated in the next theorem. Note a lower bound on $F(x_N, x_N)$ leads directly to a lower bound on $J(\pi_N)$  for $\pi_N$  parametrized by $x_N$. 
\begin{theorem} \label{th:lower bound}
There is a problem such that running \aggrevate for $N$ iterations results in $F(x_N, x_N) \geq \tilde{\epsilon}_{\Pi,\pi^*} + \Omega(N^{2(\theta-1)})$. In particular, if $\theta > 1$, the policy sequence and performance sequence diverge.
\end{theorem}

\begin{proofatend}
Consider the example in Section~\ref{sec:convergence?}. For this problem, $T=2$, $J(x^*) = 0$, and $\tilde{\epsilon}_{\Pi, \pi^*}  = 0$, implying $F(x,x) = \frac{1}{2}J(x) = \frac{1}{2} (\theta-1)^2 x^2 $. Therefore, to prove the theorem, we focus on the lower bound of $x_N^2$.

 Since $x_n = \argmin_{x \in \XX} f_{1:n-1}(x)$ and the cost is quadratic, we can write
\begin{align*}
x_{n+1} &= \argmin_{x \in \XX} f_{1:n}(x) \\ 
&= \argmin_{x \in \XX} (n-1)(x -x_n )^2 + (x - \theta x_n)^2 \\
&=  ( 1 - \frac{1-\theta }{n} ) x_n
\end{align*}
If $\theta =1$, then $x_N = x_1$ and  the bound holds trivially. 
For general cases, let $p_n = \ln (x_n^2)$. 
\begin{align*}
p_N - p_2 &= 2 \sum_{n=2}^{N-1} \ln\left( 1 - \frac{1-\theta }{n} \ \right) \\
&\geq -2 (1-\theta)\sum_{n=2}^{N-1} \frac{1}{n-(1-\theta)}
\end{align*}
where the inequality is due to the fact that $\ln(1-x) \geq \frac{-x}{1-x}$ for $x<1$. 
We consider two scenarios. Suppose $\theta < 1$. 
\begin{align*}
p_N - p_2 &\geq -2 (1-\theta) \int_{1}^{N-1} \frac{1}{x-(1-\theta)} dx\\
&= -2 (1-\theta) \ln(x-(1-\theta)) |_{1}^{N-1}  \\
&= -2 (1-\theta) \left( \ln(N+\theta-2) - \ln(\theta)\right) \\
&\geq -2 (1-\theta) \ln(N+\theta-2)
\end{align*}
Therefore, $x_N^2 \geq x_2^2 (N+\theta-2)^{2(\theta -1)} \geq  \Omega(N^{2(\theta-1)})$.

On the other hand, suppose $\theta > 1$.
\begin{align*}
p_N - p_2 &\geq 2 (\theta-1) \int_{2}^{N} \frac{1}{x-(1-\theta)} dx\\
&= 2 (\theta-1)  \ln(x-(1-\theta)) |_{2}^{N}  \\
&= 2 (\theta-1)  \left( \ln(N-1 + \theta) - \ln(1+\theta)\right)
\end{align*} 
Therefore,  $x_N^2 \geq x_2^2 (N-1+\theta)^{2(\theta -1)} (1+\theta)^{-2(\theta -1)}  \geq  \Omega(N^{2(\theta-1)})$.
Substituting the lower bound on $x_N^2$ into the definition of $F(x,x)$ concludes the proof. 
\end{proofatend}
\ifLONG\else
\begin{proof}
The proof is based on analyzing the sequence in the example in Section~\ref{sec:convergence?}. See Appendix~\ref{sec:proofs}.
\end{proof}
\fi

\subsection{Proof of Theorem~\ref{th:convergence of the last iterate}}

Now we give the proof of Theorem~\ref{th:convergence of the last iterate}. Without using the first-order information of $F$ in the first argument, we construct our analysis based on the convergence of an intermediate quantity, which indicates how fast the sequence concentrates toward its last element:
\begin{align}
S_n \coloneqq \frac{\sum_{k=1}^{n-1} \norm{ x_n - x_k}  }{n-1} \label{eq:definition of S_n}
\end{align}
which is defined $n\geq 2$ and $S_2 = \norm{x_2 - x_1}$.

First, we use Assumption~\ref{as:uniform continous gradient2} to strengthen the bound $\norm{x_{n+1} - x_n } = O(\frac{1}{n})$ used in Theorem~\ref{th:aggrevate classical result} by techniques from online learning with prediction~\citep{rakhlin2013online}.
\begin{lemma} \label{lm:bound on consecutive iterates}
	Under Assumptions~\ref{as:structured function} and \ref{as:uniform continous gradient2}, running \aggrevate gives, for $n\geq2$, 
	$
	\norm{ x_{n+1} -x_n} \leq   \frac{\theta S_n}{n}
	$.
\end{lemma}
\begin{proof}
	First, because $f_{1:n}$ is $n\alpha$-strongly convex,
	\begin{align*}
	\frac{n \alpha}{2} \norm{x_{n+1} - x_n}^2 \leq f_{1:n}(x_{n}) - f_{1:n}(x_{n+1})\\ \leq \lr{\nabla f_{1:n}(x_n) }{x_n - x_{n+1}} - \frac{n \alpha }{2} \norm{x_n - x_{n+1}}^2 .
	\end{align*}
	Let $\bar{f}_{n} = \frac{1}{n} f_{1:n}$. The above inequality implies
	\begin{align*}
	 n \alpha \norm{x_{n+1} - x_n}^2 
	&\leq  \lr{\nabla f_{n}(x_n) }{x_n - x_{n+1}}  \\ 
	&\leq \lr{\nabla f_{n}(x_n) -  \nabla \bar{f}_{n-1}(x_n) }{x_n - x_{n+1} } \\
	&\leq \norm{\nabla f_{n}(x_n) -  \nabla \bar{f}_{n-1}(x_n) }\norm{x_n - x_{n+1} }  \\
	&\leq \beta S_n \norm{x_n - x_{n+1} } 
	\end{align*}
	where the second inequality is due to $x_n = \argmin_{x \in \XX } f_{1:n-1}(x)$ and the last inequality  is due to Assumption~\ref{as:uniform continous gradient2}.
	Thus, $\norm{x_n - x_{n+1} } \leq \frac{\beta S_n}{\alpha n}$. \qedhere		
\end{proof}

Using the refined bound provided by Lemma~\ref{lm:bound on consecutive iterates}, we can bound the progress of $S_n$. 

\begin{proposition} \label{pr:bound on S_n}
	Under the assumptions in Lemma~\ref{lm:bound on consecutive iterates}, for $n \geq 2$, $ S_n \leq  e^{1-\theta} n^{\theta -1}  S_2 $ and $S_2 = \norm{x_2 - x_1} \leq \frac{G_2}{\alpha}$.
\end{proposition}
\begin{proof}
	The bound on $S_2 = \norm{x_2 - x_1}$ is due to that $x_2 = \argmin_{x\in \XX} f_1(x)$ and that $f_1$ is $\alpha$-strongly convex and $G_2$-Lipschitz continuous. 
	
	To bound $S_n$, first we bound $S_{n+1}$ in terms of $S_n$ by
	\begin{align*}
	S_{n+1} &\leq  \left( 1- \frac{1}{n} \right) S_n + \norm{x_{n+1}-x_n} \\
	&\leq \left( 1- \frac{1}{n} + \frac{\theta}{n} \right) S_n = \left(1 - \frac{1-\theta}{n}\right) S_n
	\end{align*}
	in which the first in equality is due to triangular inequality (i.e. $\norm{x_k - x_{n+1}} \leq \norm{x_k - x_{n}} + \norm{x_n - x_{n+1}}$) and the second inequality is due to Lemma~\ref{lm:bound on consecutive iterates}. 
	Let $P_n = \ln S_n$. Then we can bound
	$
	P_n - P_2 \leq \sum_{k=2}^{n-1} \ln\left(  1- \frac{1-\theta}{k}  \right)
	\leq   \sum_{k=2}^{n-1} - \frac{1-\theta}{k} \leq -(1-\theta) \left( \ln n -1 \right)
	$,
	where we use the facts that $\ln(1+x) \leq x$, $\sum_{k=1}^{n} \frac{1}{k} \geq \ln(n+1)$. 
	This implies 
	$
	S_n = \exp(P_n) \leq  e^{1-\theta} n^{\theta -1}S_2. \qedhere
	$
\end{proof}
More generally, define 	$
S_{m:n} = \frac{\sum_{k=m}^{n-1} \norm{x_n - x_k} }{n-m}
$ (i.e. $S_{n} = S_{1:n}$). Using Proposition~\ref{pr:bound on S_n}, we give a bound on $S_{m:n}$. We see that the convergence of $S_{m:n}$ depends mostly on $n$ not $m$. 
\ifLONG \else 
(The proof is given in Appendix.)
\fi

\begin{corollary}  \label{co:bound on S_mn}
	Under the assumptions in Lemma~\ref{lm:bound on consecutive iterates}, for  $n > m$, $S_{m:n} \leq O(\frac{\theta}{(n-m)m^{2-\theta}} + \frac{1}{n^{1-\theta}})$.
\end{corollary}
\begin{proofatend} 
	To prove the corollary, we introduce a basic lemma
	\begin{restatable}{lemma}{DTdynamics}
		\label{lm:convergence DT dynamical system}
		{\citep[Lemma 1]{lan2013complexity}}
		Let $\gamma_k \in (0,1)$, $k =1, 2, \dots$ be given. If the sequence $\{ \Delta_k \}_{k\geq 0 }$ satisfies
		\[
		\Lambda_{k+1} \leq (1 - \gamma_k ) \Lambda_{k} + B_k,
		\]
		then 
		\[
		\Lambda_k \leq \Gamma_k + \Gamma_k \sum_{i=1}^{k} \frac{B_i}{\Gamma_{i+1}}
		\]
		where $\Gamma_1 = \Lambda_1$ and $
		\Gamma_{k+1} = (1- \gamma_k) \Gamma_{k}
		$. 
	\end{restatable}
	
	To bound the sequence $	S_{m:n+1}$, we first apply Lemma~\ref{lm:bound on consecutive iterates}. Fixed $m$, for any $n \geq m + 1$, we have
	\begin{align*}
	S_{m:n+1} 
	&\leq \left( 1- \frac{1}{n-m+1} \right) S_{m:n} + \norm{x_{n+1}-x_n} \\
	&\leq \left( 1- \frac{1}{n-m+1} \right) S_{m:n} + \frac{\theta}{n} S_n \\
	&\leq \left( 1- \frac{1}{n-m+1} \right) S_{m:n} +  \frac{\theta c}{n^{2-\theta}}  
	\end{align*}
	where $c = S_2 e^{1-\theta}$. 
	
	Then we apply Lemma~\ref{lm:convergence DT dynamical system}. Let $k = n-m+1$ and define $R_k = S_{m:m+k-1} = S_{m:n}$ for $k \geq 2$. Then we rewrite the above inequality as 
	\begin{align*}
	R_{k+1} \leq \left( 1- \frac{1}{k} \right) R_{k} +    \frac{\theta c}{(k+m-1)^{2-\theta}}
	\end{align*}
	and define
	\[
	\Gamma_k \coloneqq \begin{cases}
	1, & k=1 \\
	(1- \frac{1}{k-1}) \Gamma_{k-1}, & k\geq 2
	\end{cases}
	\]
	By Proposition~\ref{pr:bound on S_n}, the above conversion implies for some positive constant $c$,  \[ R_2 = S_{m:m+1} = \norm{x_{m+1} - x_m } \leq  \frac{ \theta S_{m}}{m} \leq \frac{\theta c}{m^{2-\theta}} \] and $\Gamma_k \leq O(1/k)$ and $\frac{\Gamma_k}{\Gamma_i}\leq O(\frac{i}{k})$. Thus,  by Lemma~\ref{lm:convergence DT dynamical system}, we can derive
	\begin{align*}
	R_{k} &\leq \frac{1}{k} R_2 + O\left( \theta c \sum_{i=1}^{k}\frac{i}{k} \frac{1}{(i+m-1)^{2-\theta}} \right) \\
	&\leq \frac{1}{k} R_2 +  O\left( \frac{\theta c}{k}\frac{k}{\theta} \frac{1}{(m+k-1)^{1-\theta}} \right) \\
	&= \frac{1}{k} R_2 +   O\left(\frac{1}{(m+k-1)^{1-\theta}} \right) \\
	&\leq  \frac{1}{k} \frac{\theta c}{m^{2-\theta}} +   O\left( \frac{1}{(m+k-1)^{1-\theta}} \right)
	 = O(\frac{1}{n^{1-\theta}}) 
	\end{align*}
	where we use the following upper bound in the second inequality
	\begin{align*}
	\sum_{i=1}^{k}\frac{i}{(i+m-1)^{2-\theta}} 
	&\leq  \int_{0}^{k} \frac{x}{(x+m-1)^{2-\theta}} dx\\
	&=  \left. \frac{ m + (1-\theta)x -1  }{\theta (1-\theta)(m+x-1)^{1-\theta}} \right\rvert_{0}^k \\
	&=  \frac{ (1-\theta) k + m -1 }{\theta (1-\theta)(m+k-1)^{1-\theta}}  - \frac{ m-1 }{\theta (1-\theta)(m-1)^{1-\theta}} \\
	&= \frac{k}{\theta} \frac{1}{(m+k-1)^{1-\theta}} + \frac{m-1}{\theta(1-\theta)} \left( \frac{1}{(m+k-1)^{1-\theta}} - \frac{1}{(m-1)^{1-\theta}}   \right) \\
	&\leq \frac{k}{\theta} \frac{1}{(m+k-1)^{1-\theta}}  \qedhere
	\end{align*}
\end{proofatend}

Now we are ready prove Theorem~\ref{th:convergence of the last iterate} by using the concentration of $S_n$ in Proposition~\ref{pr:bound on S_n}. 
\begin{proof}[Proof of Theorem~\ref{th:convergence of the last iterate}]
First, we prove the bound on $F(x_N, x_N)$. 
Let $x_n^* \coloneqq \argmin_{x \in \XX} f_n(x)$ and let $\bar{f}_{n} = \frac{1}{n} f_{1:n}$. Then by $\alpha$-strongly convexity of $f_n$,
\begin{align*}
&f_n(x_n) - \min_{x \in \XX} f_n(x) \\
&\leq \lr{\nabla f_n(x_n)  }{ x_n - x_n^*} - \frac{\alpha}{2} \norm{x_n - x_n^*}^2 \\
&\leq \lr{\nabla f_n(x_n) - \bar{f}_{n-1}(x_n)  }{ x_n - x_n^*} - \frac{\alpha}{2} \norm{x_n - x_n^*}^2 \\
&\leq \norm{\nabla f_n(x_n) - \bar{f}_{n-1}(x_n)}_*\norm{ x_n - x_n^*} - \frac{\alpha}{2} \norm{x_n - x_n^*}^2 \\
&\leq\frac{\norm{\nabla f_n(x_n) - \bar{f}_{n-1}(x_n)}_*^2}{2\alpha} \leq \frac{\beta^2 }{2\alpha} S_n^2 
\end{align*}
where the second inequality uses the fact that $x_n = \argmin_{x \in \XX} \bar{f}_{n-1}(x)$, the second to the last inequality takes the maximum over $\norm{x_n - x_n^*}$, and the last inequality uses Assumption~\ref{as:uniform continous gradient2}.
Therefore, to bound $f_n(x_n)$, we can use Proposition~\ref{pr:bound on S_n} and Assumption~\ref{as:relaxed good approximator in X}: 
\begin{align*}
f_n(x_n) &\leq \min_{x \in \XX} f_n(x) +  \frac{\beta^2 }{2\alpha} S_n^2  \\
&\leq \tilde{\epsilon}_{\Pi,\pi^*} + \frac{\beta^2 }{2\alpha}  \left(e^{1-\theta} n^{\theta-1} \frac{G_2}{\alpha}\right)^2
\end{align*}
Rearranging the terms gives the bound in Theorem~\ref{th:convergence of the last iterate}, and that  $\norm{x_n - \bar{x}_n} \leq S_n$ gives the second result.

Now we show the convergence of $\{x_n\}$ under the condition $\theta <1$. 
It is sufficient to show that $ \lim_{n\to \infty} \sum_{k=1}^{n} \norm{x_k - x_{k+1}} < \infty$. To see this, we apply Lemma~\ref{lm:bound on consecutive iterates} and Proposition~\ref{pr:bound on S_n}: for $\theta < 1$, 
$
\sum_{k=1}^{n} \norm{x_k - x_{k+1}} 
\leq \norm{x_1 - x_2} +  \sum_{k=2}^{n} \frac{\theta}{k} S_k 
\leq c_1 + c_2 \sum_{k=2}^{n} \frac{\theta}{k}  \frac{S_2}{k^{1-\theta}}< \infty
$, where $c_1, c_2 \in O(1)$. 
\end{proof}

\subsection{Stochastic Problems}

We analyze the convergence of \aggrevate in stochastic problems using finite-sample approximation:
Define $f(x;s) = \E_{\pi}[ A_{\pi^*|t}  ]$ such that $f_n(x) = \E_{d_{\pi_n}} [f(x;s)]$. 
Instead of using $f_n(\cdot)$ as the per-round cost in the $n$th iteration, we take its finite samples approximation $g_n(\cdot) = \sum_{k=1}^{m_n} f(\cdot;s_{n,k})$, where $m_n$ is the number of independent samples collected in the $n$th iteration under distribution $d_{\pi_n}$. That is, the update rule in~\eqref{eq:AggreVate udpate} in stochastic setting is modified to
$ \pi_{n+1} = \argmin_{\pi \in \Pi} g_{1:n}(\pi) $.

\begin{restatable}{theorem}{thmStochastic} \label{th:convergence of the last iterate (stochastic)}	
	In addition to Assumptions~\ref{as:uniform continous gradient2} and \ref{as:relaxed good approximator in X}, assume $f(x;s)$ is $\alpha$-strongly convex in $x$ and $\norm{f(x;s)}_* < G_2$ almost surely. Let $\theta = \frac{\beta}{\alpha}$ and suppose $m_n = m_0 n^r$ for some $r \geq 0$. For all $N > 0$, with probability at least $1-\delta$,
	\begin{align*}
	F(x_N, x_N)  &\leq \tilde{\epsilon}_{\Pi,\pi^*} 
	+ \tilde{O}\left(  \frac{\theta^2}{c} \frac{\ln(1/\delta) +  C_{\XX}  }{  N^{\min\{r,2, 2-2\theta\} } }  \right) \\	
	&\quad + \tilde{O} \left( \frac{\ln(1/\delta) + C_{\XX }}{c N^{\min\{2,1+r\}}} \right) 
	\end{align*}
	where 	$c = \frac{\alpha}{G_2^2 m_0}$ and $C_\XX$ is a constant\footnote{The constant $C_\XX$ can be thought as $\ln |\XX|$, where $|\XX|$ measures the size of $\XX$ in e.g. Rademacher complexity or covering number~\citep{mohri2012foundations}. For example, $\ln |\XX|$ can be linear in $\dim \XX$.} of the complexity of $\Pi$.
	\end{restatable}
\begin{proof}
	The proof is similar to the proof of Theorem~\ref{th:convergence of the last iterate}. To handle the stochasticity, we use a generalization of Azuma-Hoeffding inequality to vector-valued martingales~\citep{hayes2005large}  to derive a high-probability bound on $\norm{\nabla g_n (x_n)- \nabla f_n(x_n)  }_*$ and a uniform bound on $\sup_{x\in \XX}\frac{1}{n}\norm{ \nabla g_{1:n} (x)- \nabla f_{1:n}(x)  }_*$. These error bounds allow  us to derive a stochastic version of Lemma~\ref{lm:bound on consecutive iterates}, Proposition~\ref{pr:bound on S_n}, and then the performance inequality in the proof of Theorem~\ref{th:convergence of the last iterate}. See Appendix~\ref{sec:full proof of stochastic problems} for the complete proof. 
\end{proof}

The growth of sample size $m_n$ over iterations determines the main behavior of \aggrevate in stochastic problems. 
For $r=0$, compared with Theorem~\ref{th:convergence of the last iterate}, Theorem~\ref{th:convergence of the last iterate (stochastic)} has an additional constant error in $\tilde{O}(\frac{1}{m_0})$, which is comparable to the stochastic error in selecting the best policy in the classical approach. However, the error here is due to approximating the gradient $\nabla f_n$ rather than the objective function $f_n$.
 For $r>0$, by slightly taking more samples over iterations (e.g. $r = 2-2\theta $), we see the convergence rate can get closer to $\tilde{O}(N^{2-2\theta})$ as in the ideal case given by Theorem~\ref{th:convergence of the last iterate}. However, it cannot be better than $\tilde{O}(\frac{1}{N})$. Therefore, for stochastic problems, a stability constant $\theta < 1/2$ and a growing rate $r > 1$ does not contribute to faster convergence as opposed to the deterministic case in Theorem~\ref{th:convergence of the last iterate}.

Note while our analysis here is based on finite-sample approximation  $g_n(\cdot) = \sum_{k=1}^{m_n} f(\cdot;s_{n,k})$, the same technique can also be applied to the scenario in the bandit setting and another online regression problem is solved to learn $f_n(\cdot)$ as in the case considered by~\citet{ross2014reinforcement}. A discussion is given in Appendix \ref{sec:function approximation}.

The analysis given as Theorem~\ref{th:convergence of the last iterate (stochastic)} can be viewed as a generalization of the analysis of Empirical Risk Minimization (ERM) to non-i.i.d. scenarios, where the distribution depends on the decision variable. For optimizing a strongly convex objective function with i.i.d. samples, it has been shown by~\citet{shalev2009stochastic} that $x_N$ exhibits a fast convergence to the optimal performance in $O(\frac{1}{N})$.
By specializing our general result in Theorem~\ref{th:convergence of the last iterate (stochastic)} with $\theta,r=0$ to recover the classical i.i.d. setting, we arrive at a bound on the performance of $x_N$ in $\tilde{O}(\frac{1}{N})$, which matches the best known result up to a log factor. However, Theorem~\ref{th:convergence of the last iterate (stochastic)} is proved by a completely different technique using the martingale concentration of the gradient sequence.  In addition, by Theorem~\ref{th:convergence of the last iterate}, the theoretical results of $x_N$ here can directly translate to that of the mean policy $\bar{x}_N$, which matches the bound for the average decision $\bar{x}_N$ given by \citet{kakade2009generalization}.

\section{REGULARIZATION} \label{sec:regularization}
We have shown that whether \aggrevate generates a convergent policy sequence and  a last policy with the desired performance depends on the stability constant $\theta$. Here we show that by adding regularization to the problem we can make the problem stable. For simplicity, here we consider deterministic problems or stochastic problems with infinite samples.

\subsection{Mixing Policies}
We first consider the idea of using  mixing policies to collect samples, which was originally proposed as a heuristic by~\citet{ross2011reduction}. It works as follows:  in the $n$th iteration of \aggrevate, instead of using $F(\pi_n,\cdot)$ as the per-round cost, it uses $\hat{F}(\pi_n, \cdot)$ which is defined by 
\begin{align}
\hat{F}(\pi_n, \pi) = \E_{d_{\tilde{\pi}_n}} \E_{\pi} [A_{\pi^*|t}] 
 \label{eq:mixing regularization}
\end{align} 
The state distribution $d_{\tilde{\pi}_n}(s)$ is generated by running  $\pi^*$ with probability $q$ and $\pi_n$ with probability $1-q$ at each time step. Originally, \cite{ross2011reduction} proposes to set $q$ to decay exponentially over the iterations of \aggrevate. 
\ifLONG\else
(The proofs are  given in Appendix~\ref{sec:proofs}.)
\fi

Here we show that the usage of mixing policies also has the effect of stabilizing the problem. 
\begin{lemma} \label{lm:mixing policy}
	Let $\norm{p_1 - p_2}_1$ denote the total variational distance between distributions $p_1$ and $p_2$. 
	Assume\footnote{These two are sufficient to Assumption~\ref{as:uniform Lipchitz2} and \ref{as:uniform continous gradient2}.}
	for any policy $\pi, \pi'$ parameterized by $x,y$ it satisfies
	$
	\frac{1}{T}\sum_{t=0}^{T-1}  \norm{ d_{\pi|t} - d_{\pi'|t} }_1 \leq \frac{\beta}{2G_2} \norm{x-y}  
	$
	and assume $\norm{ \nabla_x \E_{\pi} [ A_{\pi^*|t}  ](s)}_* \leq G_2$.Then $\nabla_2 F$ is uniformly $(1-q^T) \beta$-Lipschitz continuous in the second argument. 
\end{lemma}
\begin{proofatend}
	Define $\delta_{\pi|t}$ such that $d_{\pi|t;q}(s) = (1-q^t) \delta_{\pi|t}(s) + q^t d_{\pi^*}(s)$, and define $g_{z|t}(s) = \nabla_z \E_{\pi} [ Q_{\pi^*|t}  ](s) $, for $\pi$ parametrized by $z$; then by assumption, $\norm{g_{z|t}}_* < G_2$. 
	Let $\pi$, $\pi'$ be two policies parameterized by $x,y\in\XX$, respectively. Then
	\begin{align*}
	&\norm{\nabla_2 \hat{F}(x,z) - \nabla_2 \hat{F}(y,z)}_* \\
	&= \norm{\E_{d_{\tilde{\pi}}} [g_{z|t}] - \E_{d_{\tilde{\pi}'}} [g_{z|t}] }_*\\
	&= \norm{\frac{1}{T}\sum_{t=0}^{T-1} (1-q^{t})  (\E_{  \delta_{\pi|t;q}} [g_{z|t} ] - \E_{\delta_{\pi'|t;q}} [g_{z|t} ] )}_*\\
	&\leq (1-q^{T})\frac{1}{T}\sum_{t=0}^{T-1} \norm{  \E_{  \delta_{\pi|t;q}} [g_{z|t} ] - \E_{\delta_{\pi'|t;q}} [g_{z|t} ] }_*\\
	&\leq (1-q^{T})\frac{2 G_2}{T}\sum_{t=0}^{T-1} \norm{  \delta_{\pi|t;q} -\delta_{\pi'|t;q}}_1\\
	&\leq (1-q^{T})\frac{2G_2}{T}\sum_{t=0}^{T-1} \norm{   d_{\pi|t} - d_{\pi'|t}  }_1\\
	&\leq (1-q^T ) \beta \norm{x-y}
	\end{align*}
	in which the second to the last inequality is because the divergence between  $d_{\pi|t}$ and $d_{\pi'|t}$ is the largest among all state distributions generated by the mixing policies.
\end{proofatend}

By Lemma~\ref{lm:mixing policy}, if $\theta > 1$, then choosing a fixed $q  >  (1-\frac{1}{\theta} )^{1/T} $ ensures the stability constant of $\hat{F}$ to be $\hat{\theta} < 1$. However, stabilizing the problem in this way incurs a constant cost as shown in Corollary~\ref{cr:performance with mixing policy}.

\begin{corollary} \label{cr:performance with mixing policy}
	Suppose $\E_{\pi}[A_{\pi^*|t}] < M$ for all $\pi$. Define $\Delta_N =   \frac{( \hat{\theta} e^{1-\hat{\theta}}  G_2 )^2}{2\alpha}  N^{2(\hat{\theta} -1)}$. Then under the assumptions in Lemma~\ref{lm:mixing policy} and Assumption~\ref{as:uniform strongly convex2}, running \aggrevate with $\tilde{F}$ in \eqref{eq:mixing regularization} and a mixing rate $q$ gives
	\begin{align*}
	F(x_N, x_N) \leq \Delta_N + \tilde{\epsilon}_{\Pi,\pi^*} + 2M \min(1, Tq)
	\end{align*}
\end{corollary}
\begin{proofatend}
The proof is similar to Lemma~\ref{lm:mixing policy} and the proof of~\cite[Theorem 4.1]{ross2011reduction}.
\end{proofatend}

\subsection{Weighted Regularization}

Here we consider another scheme for stabilizing the problem.
Suppose $F$ satisfies Assumption~\ref{as:structured function} and \ref{as:uniform continous gradient2}. For some $\lambda >0$, define
\begin{align}
	\tilde{F}(x,x) = F(x,x) + \lambda R(x)  \label{eq:additive regularization}
\end{align}
in which\footnote{See Appendix~\ref{sec:general weighted regularization} for discussion of the case where $R(\cdot)= F(\pi^*,\cdot)$ regardless of the condition $R(x) \geq 0$.} $R(x)$ is an $\alpha$-strongly convex regularization term such that $R(x)\geq 0$, $\forall x \in \XX$ and 
$\min_{y\in \XX}F(x,y) + \lambda R(y) = (1+\lambda) O( \tilde{\epsilon}_{\Pi,\pi^*})$.
For example, $R$ can be $F(\pi^*,\cdot)$ when $\pi^*$ is (close) to optimal (e.g. in the case of \dagger), or $R(x) = \E_{s,t \sim d_{\pi^*}} \E_{a \sim \pi} \E_{a^* \sim \pi^*}[d(a,a^*) ]$, where $\pi$ is a policy parametrized by $x$ and $d(\cdot,\cdot)$ is some metric of space $\Abb$ (i.e. it uses the distance between $\pi$ and $\pi^*$ as regularization).

It can be seen that $\tilde{F}$ is uniformly $(1+\lambda)\alpha$-strongly convex in the second argument and $\nabla_2 \tilde{F}$ is uniformly $\beta$-continuous in the second argument. That is, if we choose $\lambda > \theta-1$, then the stability constant $\tilde{\theta}$ of  $\tilde{F}$ satisfies $\tilde{\theta} < 1$.

\begin{corollary} \label{cr:performance with weighted regularization}
	Define 
	$\Delta_N =  \frac{( \tilde{\theta} e^{1-\tilde{\theta}}  G_2 )^2}{2\alpha}  N^{2(\tilde{\theta} -1)}.$
	Running \aggrevate with $\tilde{F}$ in~\eqref{eq:additive regularization} as the per-round cost has performance satisfies: for all $N >0$,
	\begin{align*}
	F(x_N, x_N) &\leq (1+\lambda)\left(  O( \tilde{\epsilon}_{\Pi,\pi^*})  + \Delta_N \right) 
	\end{align*}
\end{corollary}
\begin{proof}
	Because 
	$	F(x_N, x_N) = \tilde{F}(x_N, x_N) - \lambda R(x_N) $, the inequality can be proved by applying Theorem~\ref{th:convergence of the last iterate} to $\tilde{F}(x_N, x_N)$.
\end{proof}
By Corollary~\ref{cr:performance with weighted regularization}, using \aggrevate to solve a weighted regularized problem in~\eqref{eq:additive regularization} would generate a convergent sequence for $\lambda$ large enough. Unlike using a mixing policy, here the performance guarantee on the last policy is only worsened by a multiplicative constant on $ \tilde{\epsilon}_{\Pi,\pi^*}$, which can be made small by choosing a larger policy class. 

The result in Corollary~\ref{cr:performance with weighted regularization} can be strengthened particularly when $R(x) = \E_{s,t \sim d_{\pi^*}} \E_{a \sim \pi} \E_{a^* \sim \pi^*}[d(a,a^*) ]$ is used. 
In this case, it can be shown that $C R(x) \geq F(x,x)$  for some $C > 0$ (usually $C>1$)~\citep{pan2017agile}. That is,  $F(x,x) + \lambda R(x) \geq (1+\lambda/C) F(x,x)$. Thus, the multiplicative constant in Corollary~\ref{cr:performance with weighted regularization} can be reduced from $1+\lambda$ to $\frac{1+\lambda}{1+\lambda/C}$.
It implies that simply by adding a portion of demonstrations gathered under the expert's distribution so that the leaner can anchor itself to the expert while minimizing $F(x,x)$, one does not have to find  the best policy in the sequence $\{ \pi_n \}_{n=1}^N$ as in~\eqref{eq:AggreVate search}, but just return the last policy $\pi_N$.

\section{CONCLUSION} \label{sec:conclusion}

We contribute a new analysis of value aggregation, unveiling several interesting theoretical insights. Under a weaker assumption than the classical result, we prove that the convergence of the last policy depends solely on a problem's structural property and we provide a tight non-asymptotic bound on its performance in both deterministic and stochastic problems. In addition, using the new theoretical results, we show that the stability of the last policy can be reinforced by additional regularization with minor performance loss. This suggests that under proper conditions a practitioner can just run \aggrevate and then take the last policy, without performing an additional statistical test to find the best policy, as required by the classical analysis. 
Finally, as our results concerning the last policy are based on the perturbation of gradients, we believe this provides a potential explanation as to why \aggrevate has demonstrated empirical success in non-convex problems with neural-network policies.

\subsubsection*{Acknowledgments}
This work was supported in part by NSF NRI award 1637758.

\bibliographystyle{apalike}
\bibliography{ref}

\begin{thebibliography}{}

\bibitem[Antsaklis and Michel, 2007]{antsaklis2007linear}
Antsaklis, P.~J. and Michel, A.~N. (2007).
\newblock {\em A linear systems primer}, volume~1.
\newblock Birkh{\"a}user Boston.

\bibitem[Bertsekas et~al., 1995]{bertsekas1995dynamic}
Bertsekas, D.~P., Bertsekas, D.~P., Bertsekas, D.~P., and Bertsekas, D.~P.
  (1995).
\newblock {\em Dynamic programming and optimal control}, volume~1.
\newblock Athena scientific Belmont, MA.

\bibitem[Cesa-Bianchi et~al., 2004]{cesa2004generalization}
Cesa-Bianchi, N., Conconi, A., and Gentile, C. (2004).
\newblock On the generalization ability of on-line learning algorithms.
\newblock {\em IEEE Transactions on Information Theory}, 50(9):2050--2057.

\bibitem[Cucker and Zhou, 2007]{cucker2007learning}
Cucker, F. and Zhou, D.~X. (2007).
\newblock {\em Learning theory: an approximation theory viewpoint}, volume~24.
\newblock Cambridge University Press.

\bibitem[Hayes, 2005]{hayes2005large}
Hayes, T.~P. (2005).
\newblock A large-deviation inequality for vector-valued martingales.
\newblock {\em Combinatorics, Probability and Computing}.

\bibitem[Hazan et~al., 2016]{hazan2016introduction}
Hazan, E. et~al. (2016).
\newblock Introduction to online convex optimization.
\newblock {\em Foundations and Trends{\textregistered} in Optimization},
  2(3-4):157--325.

\bibitem[Kakade and Langford, 2002]{kakade2002approximately}
Kakade, S. and Langford, J. (2002).
\newblock Approximately optimal approximate reinforcement learning.
\newblock In {\em International Conference on Machine Learning}, volume~2,
  pages 267--274.

\bibitem[Kakade and Tewari, 2009]{kakade2009generalization}
Kakade, S.~M. and Tewari, A. (2009).
\newblock On the generalization ability of online strongly convex programming
  algorithms.
\newblock In {\em Advances in Neural Information Processing Systems}, pages
  801--808.

\bibitem[Lan, 2013]{lan2013complexity}
Lan, G. (2013).
\newblock The complexity of large-scale convex programming under a linear
  optimization oracle.
\newblock {\em arXiv preprint arXiv:1309.5550}.

\bibitem[Laskey et~al., 2017]{laskey2017comparing}
Laskey, M., Chuck, C., Lee, J., Mahler, J., Krishnan, S., Jamieson, K., Dragan,
  A., and Goldberg, K. (2017).
\newblock Comparing human-centric and robot-centric sampling for robot deep
  learning from demonstrations.
\newblock In {\em IEEE International Conference on Robotics and Automation},
  pages 358--365. IEEE.

\bibitem[Lee et~al., 1998]{lee1998importance}
Lee, W.~S., Bartlett, P.~L., and Williamson, R.~C. (1998).
\newblock The importance of convexity in learning with squared loss.
\newblock {\em IEEE Transactions on Information Theory}, 44(5):1974--1980.

\bibitem[McMahan, 2014]{mcmahan2014survey}
McMahan, H.~B. (2014).
\newblock A survey of algorithms and analysis for adaptive online learning.
\newblock {\em arXiv preprint arXiv:1403.3465}.

\bibitem[Mnih et~al., 2013]{mnih2013playing}
Mnih, V., Kavukcuoglu, K., Silver, D., Graves, A., Antonoglou, I., Wierstra,
  D., and Riedmiller, M. (2013).
\newblock Playing atari with deep reinforcement learning.
\newblock {\em arXiv preprint arXiv:1312.5602}.

\bibitem[Mohri et~al., 2012]{mohri2012foundations}
Mohri, M., Rostamizadeh, A., and Talwalkar, A. (2012).
\newblock {\em Foundations of machine learning}.
\newblock MIT press.

\bibitem[Nesterov, 2013]{nesterov2013introductory}
Nesterov, Y. (2013).
\newblock {\em Introductory lectures on convex optimization: A basic course},
  volume~87.
\newblock Springer Science \& Business Media.

\bibitem[Pan et~al., 2017]{pan2017agile}
Pan, Y., Cheng, C.-A., Saigol, K., Lee, K., Yan, X., Theodorou, E., and Boots,
  B. (2017).
\newblock Agile off-road autonomous driving using end-to-end deep imitation
  learning.
\newblock {\em arXiv preprint arXiv:1709.07174}.

\bibitem[Pomerleau, 1989]{pomerleau1989alvinn}
Pomerleau, D.~A. (1989).
\newblock Alvinn: An autonomous land vehicle in a neural network.
\newblock In {\em Advances in Neural Information Processing Systems}, pages
  305--313.

\bibitem[Rakhlin and Sridharan, 2013]{rakhlin2013online}
Rakhlin, A. and Sridharan, K. (2013).
\newblock Online learning with predictable sequences.
\newblock In {\em Conference on Learning Theory}, pages 993--1019.

\bibitem[Ross and Bagnell, 2014]{ross2014reinforcement}
Ross, S. and Bagnell, J.~A. (2014).
\newblock Reinforcement and imitation learning via interactive no-regret
  learning.
\newblock {\em arXiv preprint arXiv:1406.5979}.

\bibitem[Ross et~al., 2011]{ross2011reduction}
Ross, S., Gordon, G.~J., and Bagnell, D. (2011).
\newblock A reduction of imitation learning and structured prediction to
  no-regret online learning.
\newblock In {\em International Conference on Artificial Intelligence and
  Statistics}, pages 627--635.

\bibitem[Ross et~al., 2013]{ross2013learning}
Ross, S., Melik-Barkhudarov, N., Shankar, K.~S., Wendel, A., Dey, D., Bagnell,
  J.~A., and Hebert, M. (2013).
\newblock Learning monocular reactive uav control in cluttered natural
  environments.
\newblock In {\em IEEE International Conference onRobotics and Automation},
  pages 1765--1772. IEEE.

\bibitem[Shalev-Shwartz et~al., 2009]{shalev2009stochastic}
Shalev-Shwartz, S., Shamir, O., Srebro, N., and Sridharan, K. (2009).
\newblock Stochastic convex optimization.
\newblock In {\em Conference on Learning Theory}.

\bibitem[Silver et~al., 2016]{silver2016mastering}
Silver, D., Huang, A., Maddison, C.~J., Guez, A., Sifre, L., Van Den~Driessche,
  G., Schrittwieser, J., Antonoglou, I., Panneershelvam, V., Lanctot, M.,
  et~al. (2016).
\newblock Mastering the game of go with deep neural networks and tree search.
\newblock {\em Nature}, 529(7587):484--489.

\bibitem[Sun et~al., 2017]{sun2017deeply}
Sun, W., Venkatraman, A., Gordon, G.~J., Boots, B., and Bagnell, J.~A. (2017).
\newblock Deeply aggrevated: Differentiable imitation learning for sequential
  prediction.
\newblock {\em arXiv preprint arXiv:1703.01030}.

\bibitem[Sutton and Barto, 1998]{sutton1998introduction}
Sutton, R.~S. and Barto, A.~G. (1998).
\newblock {\em Introduction to reinforcement learning}, volume 135.
\newblock MIT Press Cambridge.

\bibitem[Vapnik, 1998]{vapnik1998statistical}
Vapnik, V.~N. (1998).
\newblock {\em Statistical learning theory}, volume~1.
\newblock Wiley New York.

\end{thebibliography}

\ifLONG \else
\newpage
\onecolumn
\appendix
\section*{\centering Appendix}

\section{Proofs} \label{sec:proofs}
\printproofs
\fi

\section{Analysis of \aggrevate in Stochastic Problems} \label{sec:full proof of stochastic problems}

Here we give the complete analysis of  the convergence of \aggrevate in stochastic problems using finite-sample approximation. For completeness, we restate the results below: 
Let $f(x;s) = \E_{\pi}[ A_{\pi^*|t}  ]$ (i.e. $f_n(x) = \E_{d_{\pi_n}} [f(x;s)]$, where policy $\pi$ is a policy parametrized by $x$.
Instead of using $f_n(\cdot)$ as the per-round cost in the $n$th iteration, we use consider its finite samples approximation $g_n(\cdot) = \sum_{k=1}^{m_n} f(\cdot;s_{n,k})$, where $m_n$ is the number of independent samples collected in the $n$th iteration.

\thmStochastic*

\subsection{Uniform Convergence of Vector-Valued Martingales}

To prove Theorem~\ref{th:convergence of the last iterate (stochastic)}, we first introduces several concentration inequalities of vector-valued martingales by~\citep{hayes2005large} in Section~\ref{sec:general AH lemma}. 
Then we prove some basic lemmas regarding the convergence the stochastic dynamical systems of $\nabla g_n(x)$ specified by \aggrevate in Section~\ref{sec:concentration of iid vv functions} and \ref{sec:concentration of stochastic process}. Finally, the lemmas in these two sections are extended to provide uniform bounds, which are required to prove Theorem~\ref{th:convergence of the last iterate (stochastic)}. In this section, we will state the results generally without limiting ourselves to the specific functions used in \aggrevate.

\subsubsection{Generalization of Azuma-Hoeffding Lemma} \label{sec:general AH lemma}

First we introduce two theorems by~\citet{hayes2005large} which extend Azuma-Hoeffding lemma to vector-valued martingales but without dependency on dimension.
\begin{theorem} \label{th:concentration of vector-valued martingale}
	\citep[Theorem 1.8]{hayes2005large}
	Let $\{ X_n \}$ be a (very-weak) vector-valued martingale such that $X_0 =0$ and for every $n$, $\norm{X_n - X_{n-1}} \leq 1$ almost surely. Then, for every $a>0$, it holds
	\begin{align*} 
	\Pr( \norm{X_n} \geq a ) < 2 e\exp\left( \frac{-(a-1)^2}{2n} \right)
	\end{align*}
\end{theorem}

\begin{theorem}  \label{th:concentration of weighted  vector-valued martingale}
	\citep[Theorem 7.4]{hayes2005large}
	Let $\{ X_n \}$ be a (very-weak) vector-valued martingale such that $X_0 =0$ and for every $n$, $\norm{X_n - X_{n-1}} \leq c_n$ almost surely. Then, for every $a>0$, it holds
	\begin{align*} 
	\Pr( \norm{X_n} \geq a ) < 2 \exp\left(\frac{-(a-Y_0)^2}{2 \sum_{i=1}^{n} c_i^2}  \right)
	\end{align*}
	where $Y_0 = \max\{ 1+ \max c_i, 2 \max c_i \}$. 
\end{theorem}

\subsubsection{Concentration of i.i.d. Vector-Valued Functions} \label{sec:concentration of iid vv functions}

Theorem~\ref{th:concentration of vector-valued martingale} immediately implies the concentration of  approximating vector-valued functions with finite samples.

\begin{lemma} \label{lm:mean vector-valued function}
	Let $x \in \XX$ and let $f(x) = \E_{\omega}[ f(x;\omega) ]$, where $f: \XX \to E$ and $E$ is equipped with norm $\norm{\cdot}$.  Assume $\norm{f( x; \omega)} \leq G$ almost surely. Let $g(x) = \frac{1}{M} \sum_{m=1}^{M} f(x; \omega_k)$ be its finite sample approximation. Then, for all $\epsilon >0$,
	\begin{align*}
	\Pr( \norm{g(x) - f(x)} \geq \epsilon   ) < 2 e \exp\left( - \frac{(\frac{ M \epsilon}{2 G} -1  )^2}{2M}\right)
	\end{align*}
	In particular, for  $0 < \epsilon \leq 2G$,
	\begin{align*}
	\Pr( \norm{g(x) - f(x)} \geq  \epsilon) < 2 e^2 \exp\left(-\frac{M \epsilon^2}{8G^2} \right)  
	\end{align*}
\end{lemma}
\begin{proof}
	Define $X_m = \frac{1}{2G}\sum_{k=1}^{m} f(x;\omega_k) - f(x)$. Then $X_m$ is vector-value martingale and $\norm{X_m - X_{m-1}} \leq 1$. By Theorem~\ref{th:concentration of vector-valued martingale}, 
	\begin{align*}
	\Pr( \norm{g(x) - f(x)} \geq \epsilon   ) = \Pr( \norm{X_M} \geq \frac{M\epsilon }{2G} ) < 2 e \exp\left( - \frac{(\frac{ M \epsilon}{2 G} -1  )^2}{2M}\right)
	\end{align*}
	Suppose $\frac{\epsilon}{2G} < 1$. Then
	$
	\Pr( \norm{X_M} \geq  \epsilon) < 2 e^2 \exp\left(-\frac{M \epsilon^2}{8G^2} \right)  
	$.
\end{proof}

\subsubsection{Concentration of the Stochastic Process of \aggrevate}
\label{sec:concentration of stochastic process}

Here we consider a stochastic process that shares the same characteristics of the dynamics of $\frac{1}{n} \nabla g_{1:n}(x)$ in \aggrevate and provide a lemma about its concentration. 

\begin{lemma} \label{lm:mean vector-valued function martingale}
	Let $n=1\dots N$ and $\{ m_i \}$ be a non-decreasing sequence of positive integers. Given $x \in \XX$,   let $ Y_n \coloneqq \{f_n(x;\omega_{n,k})\}_{k=1}^{m_n}$ be a set of random vectors in some normed space with norm $\norm{\cdot}$ defined as follows: Let $Y_{1:n} \coloneqq \{Y_k\}_{k=1}^{n}$. Given $Y_{1:n-1}$,  $\{f_n(x;\omega_{n,k})\}_{k=1}^{m_n}$ are $m_n$ independent random vectors
	such that  $f_n(x) \coloneqq \E_{\omega  }[ f_n(x;\omega) | Y_{1:n-1}]$ and $\norm{f_n( x; \omega)} \leq G$ almost surely. 
	Define $g_n(x) \coloneqq \frac{1}{m_n} \sum_{k=1}^{m_n} f_n(x; \omega_{n,k})$, and let $\bar{g}_n = \frac{1}{n} g_{1:n}$ and $\bar{f}_n = \frac{1}{n} f_{1:n}$. Then for all $\epsilon> 0 $, 
	\begin{align*} 
	\Pr( \norm{\bar{g}_n(x) - \bar{f}_n(x)} \geq \epsilon ) < 2 \exp\left(\frac{- (n M^* \epsilon  -Y_0)^2}{ 8 G^2 M^{*2} \sum_{i=1}^{n} \frac{1}{m_i}  }  \right)
	\end{align*}
	in which $M^* = \prod_{i=1}^{n} m_i$ and $Y_0 = \max \{1+  \frac{2 M^*  G}{m_0}, 2  \frac{2 M^*  G}{m_0}\}$. 
	
	In particular, if $\frac{2 M^*  G}{m_0} > 1$, for  $0< \epsilon \leq  \frac{G m_0 }{ n } \sum_{i=1}^{n} \frac{1}{m_i} $, 
	\begin{align*} 
	\Pr( \norm{\bar{g}_n(x) - \bar{f}_n(x)} \geq \epsilon ) &< 2e  \exp\left(\frac{-n^2 \epsilon^2}{  8 G^2 \sum_{i=1}^{n} \frac{1}{m_i}}  \right) 
	\end{align*}
\end{lemma}
\begin{proof}
	Let $M = \sum_{i=1}^{n} m_i$. Consider a martingale, for $m = l + \sum_{i=1}^{k-1} m_i $,
	\begin{align*}
	X_m =  \frac{M^*}{m_{k}} \sum_{i=1}^{l} f_k(x;\omega_{k,i}) - f_k(x)  +
	 \sum_{i=1}^{k-1} \frac{M^*}{m_i} \sum_{j=1}^{m_i} f_i(x;\omega_{i,j}) - f_i(x).
	\end{align*}
	That is, $X_{M} = nM^* (\bar{g}_n - \bar{f}_n) $ and $\norm{X_m - X_{m-1}} \leq \frac{2 M^* G}{m_i}$ for some appropriate $m_i$. Applying Theorem~\ref{th:concentration of weighted  vector-valued martingale}, we have 
	\begin{align*} 
	\Pr( \norm{\bar{g}_n - \bar{f}_n} \geq \epsilon ) = \Pr( \norm{X_M} \geq n M^* \epsilon )  < 2 \exp\left(\frac{-(n M^* \epsilon-Y_0)^2}{2 \sum_{m=1}^{M} c_m^2}  \right)
	\end{align*}
	where 
	\begin{align*}
	\sum_{m=1}^{M} c_m^2 = \sum_{i=1}^{n} \sum_{j=1}^{m_i} \left( \frac{2G M^* }{m_i}  \right)^2 =  4G^2 M^{*2}\sum_{i=1}^{n} \frac{1}{m_i}.
	\end{align*}
	In addition, by assumption $m_i \leq m_{i-1}$,  $Y_0 =  \max \{1+  \frac{2 M^* G}{m_0}, 2  \frac{2 M^* G}{m_0}\}$. This gives the first inequality. 
	
	For the special case, the following holds 
	\begin{align*}
	\frac{-(n M^* \epsilon-Y_0)^2}{2 \sum_{m=1}^{M} c_m^2} 
	= \frac{-n^2 M^{*2} \epsilon^2 }{ 8 G^2 M^{*2}\sum_{i=1}^{n} \frac{1}{m_i}} 
	+ \frac{ 2 n M^* \epsilon Y_0 - Y_0^2}{ 8  G^2 M^{*2}\sum_{i=1}^{n} \frac{1}{m_i} }
	\leq \frac{-n^2 \epsilon^2 }{ 4G^2 \sum_{i=1}^{n} \frac{1}{m_i}} + 1 
	\end{align*}
	if $\epsilon$ satisfies
	\begin{align*}
	2 n M^* \epsilon Y_0 < 8 G^2 M^{*2}\sum_{i=1}^{n} \frac{1}{m_i}  \implies \epsilon <  \frac{ 4G^2 M^{*} }{Y_0 n } \sum_{i=1}^{n} \frac{1}{m_i}
	\end{align*}
	Substituting the condition that $Y_0 = \frac{4 M^* G}{m_0}$ when $\frac{2 M^* G}{m_0} > 1$, a sufficient range of $\epsilon$ can be obtained as
	\begin{align*}
	\frac{ 4G^2 M^{*} }{Y_0 n } \sum_{i=1}^{n} \frac{1}{m_i} =  \frac{ G m_0 }{ n } \sum_{i=1}^{n} \frac{1}{m_i} \geq \epsilon.
	\end{align*}

\end{proof}

\subsubsection{Uniform Convergence} \label{sec:uniform convergence}

The above inequality holds for a particular $x \in \XX$. Here we use the concept of covering number to derive uniform bounds that holds for all $x \in \XX$. (Similar (and tighter) uniform bounds can also be derived using Rademacher complexity.)
\begin{definition}
	Let $S$ be a metric space and $\eta > 0$. The covering number $\NN(S,\eta)$ is the minimal $l \in \NN$ such that $S$ is is covered by $l$ balls of radius $\eta$. When $S$ is compact, $\NN(S,\eta)$ is finite.
\end{definition}

As we are concerned with vector-valued functions, let $E$ be a normed space with norm $\norm{\cdot}$.
Consider a mapping $f: \XX \to \BB$ defined as 
$
f:  x  \mapsto f(x, \cdot) 
$,
where  $\BB = \{ g: \Omega \to E \}$ is a Banach space of vector-valued functions with norm   $\norm{g}_{\BB}  = \sup_{\omega \in \Omega} \norm{g(\omega)}$. 
Assume $\BB_\XX = \{ f(x, \cdot): x\in \XX \}$ is a compact subset in $\BB$. Then the covering number of $\HH$ is finite and given as $\NN(\BB_\XX, \eta)$. That is, there exists a finite set $\CC_\XX = \{ x_{i} \in \XX \}_{i=1}^{\NN(\BB_\XX, \eta)}$ such that $\forall x \in \XX$, $\min_{y \in \CC_\XX  } \norm{f(x,\cdot) -  f(y,\cdot) }_\BB < \eta$.

Usually, the covering is a polynomial function of $\eta$. For example, suppose $\XX$ is a ball of radius $R$ in a $d$-dimensional Euclidean space, and $f$ is $L$-Lipschitz in $x$ (i.e. $\norm{f(x,\cdot) - f(y,\cdot)}_\BB \leq L \norm{x -y}$). Then \citep{cucker2007learning} 
$
\NN(\BB_\XX, \eta) \leq \NN(\XX, \frac{\eta}{L}) \leq \left( \frac{2RL}{\eta} + 1  \right)^d
$. 
Therefore, henceforth we will assume 
\begin{align}
	\ln \NN(\BB_XX, \eta) \leq C_{\XX} \ln (\frac{1}{\eta})  < \infty \label{eq:covering number assumption}
\end{align}
for some constant $C_{\XX} $ independent of $\eta$, which characterizes the complexity of $\XX$.

Using covering number, we derive uniform bounds for the lemmas in Section~\ref{sec:concentration of iid vv functions} and \ref{sec:concentration of stochastic process}.

\begin{lemma} \label{lm:mean vector-valued function (uniform)}
	Under the assumptions in Lemma~\ref{lm:mean vector-valued function}, for $0 < \epsilon \leq 2G$,
	\begin{align*}
	\Pr( \sup_{x\in \XX} \norm{g(x) - f(x)} \geq  \epsilon) < 2 e^2 \NN(\BB_\XX, \frac{\epsilon}{4}) \exp\left(-\frac{M \epsilon^2}{32G^2} \right)  
	\end{align*}
\end{lemma}
\begin{proof}
	Choose $\CC_\XX$ be the set of the centers of the covering balls such that $\forall x \in \XX$, $\min_{y \in \CC_\XX  } \norm{f(x,\cdot) -  f(y,\cdot) }_\BB < \eta$.
	Since $f(x) = \E_{\omega }[f(x,\omega)]$, it also holds  $\min_{y \in \CC_\XX  } \norm{f(x) -  f(y) } < \eta$.
	Let $B_y$ be the $\eta$-ball centered for $y \in \CC_\XX$. Then
	\begin{align*}
	\sup_{y\in \XX} \norm{g(x) - f(x)} &\leq \max_{y \in \CC_\XX}  \sup_{x\in B_y} \norm{g(x) - g(y)} + \norm{g(y) - f(y)} + \norm{f(y) - f(x)} \\
	&\leq  \max_{y \in \CC_\XX}   \norm{g(y) - f(y)} +  2 \eta
	\end{align*}
	Choose $\eta = \frac{\epsilon}{4}$ and then it follows that
	\begin{align*}
	\sup_{x\in \XX} \norm{g(x) - f(x)} \geq \epsilon \implies \max_{y \in \CC_\XX} \norm{g(y) - f(y)} \geq \frac{\epsilon}{2}
	\end{align*}
	The final result can be obtained by first for each $y\in \CC_\XX$ applying the concentration inequality with $\epsilon/2$ and then a uniform bound over $\CC_\XX$. 
\end{proof}
Similarly, we can give a uniform version of Lemma~\ref{lm:mean vector-valued function martingale}.
\begin{lemma} \label{lm:mean vector-valued function martingale (uniform)}
	Under the assumptions in Lemma~\ref{lm:mean vector-valued function martingale}, if $\frac{2 M^*  G}{m_0} > 1$, for  $ 0 < \epsilon \leq  \frac{G m_0 }{ n } \sum_{i=1}^{n} \frac{1}{m_i} $ and for a fixed $n \geq 0$,
	\begin{align*} 
	\Pr( \sup_{x \in \XX} \norm{\bar{g}_n(x) - \bar{f}_n(x)} \geq \epsilon ) &< 2e   \NN\left (\BB_\XX, \frac{\epsilon}{4}\right ) \exp\left(\frac{-n^2 \epsilon^2}{  32G^2 \sum_{i=1}^{n} \frac{1}{m_i}}  \right) 
	\end{align*}
\end{lemma}

\subsection{Proof of Theorem~\ref{th:convergence of the last iterate (stochastic)}}

We now refine Lemma~\ref{lm:bound on consecutive iterates} and Proposition~\ref{pr:bound on S_n} to prove the convergence of \aggrevate in stochastic problems. We use $\bar{\cdot}$ to denote the average (e.g. $\bar{f}_n = \frac{1}{n}f_{1:n}$.)

\subsubsection{Bound on $\norm{x_{n+1} - x_n}$}

First, we show the error due to finite-sample approximation.
\begin{lemma} \label{lm:bound on consecutive iterates (error)}   
	Let $\xi_n= \nabla  f_n-  \nabla g_n$.   Running \aggrevate with $g_n(\cdot)$ as per-round cost gives, for $n\geq2$,
	\begin{align*}
	\norm{ x_{n+1} -x_n} \leq   \frac{\theta S_n}{n} + \frac{1}{n \alpha} \left(  \norm{ \xi_n(x_n)  }_*  +  \norm{ \bar{\xi}_{n-1}(x_n)}_*  \right)
	\end{align*}
\end{lemma}
\begin{proof}		
	Because $g_{1:n}(x)$ is $n\alpha$-strongly convex in $x$, we have
	\begin{align*}
	n \alpha \norm{x_{n+1} - x_n}^2 
	&\leq  \lr{\nabla g_{n}(x_n) }{x_n - x_{n+1}}  \\
	&\leq \lr{\nabla g_{n}(x_n) -  \nabla \bar{g}_{n-1}(x_n) }{x_n - x_{n+1} } 	&\because x_n = \argmin_{x \in \XX } g_{1:n-1}(x)  \\
	&\leq \norm{\nabla f_{n}(x_n) -  \nabla \bar{f}_{n-1}(x_n) }_*\norm{x_n - x_{n+1} }  \\
	&\quad + \norm{\nabla f_{n}(x_n) - \nabla g_{n}(x_n)-  \nabla \bar{f}_{n-1}(x_n) + \nabla \bar{g}_{n-1}(x_n)  }_*\norm{x_n - x_{n+1} }  
	\end{align*}
	Now we use the fact that the smoothness applies  to $f$ (not necessarily to $g$) and derive the statement
	\begin{align*}
	\norm{x_{n+1} - x_n} &\leq \frac{\theta S_n }{n  }  + \frac{1}{n \alpha}  \norm{\nabla f_{n}(x_n) - \nabla g_{n}(x_n)-  \nabla \bar{f}_{n-1}(x_n) + \nabla \bar{g}_{n-1}(x_n)  }_* \\
	&\leq  \frac{\theta S_n }{n  } + \frac{1}{n \alpha} \left(  \norm{ \xi_n(x_n)  }_*  +  \norm{ \bar{\xi}_{n-1}(x_n)}_*  \right)   \qedhere
	\end{align*}	
\end{proof}

Given the intermediate step in Lemma~\ref{lm:bound on consecutive iterates (error)}, we apply Lemma~\ref{th:concentration of vector-valued martingale} to  bound the norm of $\xi_k$ and give the refinement of Lemma~\ref{lm:bound on consecutive iterates} for stochastic problems. 
\begin{lemma} \label{lm:bound on consecutive iterates (stochastic)}
	Suppose  $m_n = m_0 n^{r}$ for some $r \geq 0$. Under previous assumptions, running \aggrevate with $g_n(\cdot)$ as per-round cost,  the following holds with  probability at least $1-\delta$: For a fixed $n\geq2$, 
	
	\begin{align*}
	\norm{ x_{n+1} -x_n} \leq   \frac{\theta S_n}{n} +   O\left( 
	\frac{G_2}{n \alpha \sqrt{m_0}} \left( \sqrt{  \frac{  \ln(1/\delta) }{ n^{ \min\{r,2\} }}  } + \sqrt{   \frac{  C_{\XX}/n }{  n^{ \min\{r,1\}   }}    } \right) 
	\right)	
	\end{align*}
	where $C_{\XX}$ is a constant depending on the complexity of $\XX$ and the constant term in big-$O$ is some universal constant.
\end{lemma}

\begin{proof}
	To show the statement, we bound $ \norm{ \xi_n(x_n)  }_* $ and $ \norm{ \bar{\xi}_{1:n-1}(x_n)}_*  $ in Lemma \ref{lm:bound on consecutive iterates (error)} using the concentration lemmas derived in Section~\ref{sec:uniform convergence}.
	
	\textbf{\textit{The First Term:}}
	To bound $\norm{ \xi_n(x_n)  }_*$, because the sampling of $\xi_n$ is independent of  $x_n$, bounding  $\norm{ \xi_n(x_n)  }_*$ does not require a uniform bound. Here we use Lemma~\ref{lm:mean vector-valued function}
	and consider $\epsilon_1$ such that
	\begin{align}
	2 e^2 \exp\left(-\frac{ m_n \epsilon_1^2}{8G_2^2} \right) =  \frac{\delta}{2} 
	\implies  \epsilon_1 = \sqrt{ \frac{8G_2^2}{m_n} \ln\left( \frac{4e^2}{\delta}  \right)  } 
	= O \left( \sqrt{ \frac{G_2^2}{m_n} \ln\left( \frac{1}{\delta}  \right)  }  \right)
	 \label{eq:bound eps1}
	\end{align}
	Note we we used the particular range of $\epsilon$ in  Lemma~\ref{lm:mean vector-valued function} for convenience, which is valid if we choose $m_0 > 2G_2 \ln\left( \frac{4e^2}{\delta}  \right) $. This condition is not necessary; it is only used to simplify the derivation, and using a different range of $\epsilon$ would simply lead to a different constant. 
	\paragraph{\textit{The Second Term:}}
	To bound $\norm{ \bar{\xi}_{n-1}(x_n)  }_*$, we apply a uniform bound using Lemma~\ref{lm:mean vector-valued function martingale (uniform)}. For simplicity, we  use the particular range $0 < \epsilon \leq  \frac{G_2 m_0 }{ n } \sum_{i=1}^{n} \frac{1}{m_i} $ and assume  $\frac{2 M^*  G_2}{m_0} > 1$ (which implies $Y_0 = \frac{4 M^* G_2}{m_0}$) (again this is not necessary).
	We  choose $\epsilon_2$ such that 
	\begin{align*}
	2e  \NN(\BB_\XX, \frac{\epsilon_2}{4})  \exp\left(\frac{-(n-1)^2 \epsilon_2^2}{  32G_2^2  \sum_{i=1}^{n-1} \frac{1}{m_i}}  \right) \leq \frac{\delta}{2}
	\implies
	\ln (2e)  + \ln  \NN(\BB_\XX, \frac{\epsilon_2}{4})  + \frac{-(n-1)^2 \epsilon_2^2}{  32G_2^2  \sum_{i=1}^{n-1} \frac{1}{m_i}}   \leq - \ln(\frac{2}{\delta})
	\end{align*}
    Since $ \ln  \NN(\BB_\XX, \frac{\epsilon_2}{4}) =  C_{\XX} \ln\left( \frac{4}{\epsilon_2}\right) \leq  c_s C_{\XX}  \epsilon_2^{-s}$ for  arbitrary $s >  0 $ and some $c_s$,     
    a sufficient condition can be obtained by solving for $\epsilon_2$ such that
	\begin{align*}
	\frac{c_0}{\epsilon_2^s} - c_2 \epsilon_2^2 = - c_1 \implies  c_2 \epsilon_2^{2+s} - c_1\epsilon_2^s - c_0 = 0
	\end{align*}
	where $c_0 = c_s C_{\XX}$, $c_2 =  \frac{(n-1)^2}{  32G_2^2  \sum_{i=1}^{n-1} \frac{1}{m_i}} $, and $c_1 = \ln(\frac{4e}{\delta})$. 	
	To this end, we use a basic lemma of polynomials.
	\begin{lemma}\cite[Lemma 7.2]{cucker2007learning} \label{lm:a basic lemma for solution}		
	Let $c_1, c_2, \dots, c_l > 0$ and $ s > q_1 > q_2 > \dots > q_{l-1} >  0$. Then the equation
	\begin{align*}
		x^s - c_1 x^{q_1} - c_2 x^{q_2} - \dots - c_{l-1} x^{q_{l-1}} - c_l = 0 
	\end{align*}
	has a unique solution $x^*$. In addition, 
	\begin{align*}
		x^* \leq \max\left\{ (lc_1)^{1/(s-q_1)}, (lc_2)^{1/(s-q_2)}, \dots, (lc_{l-1})^{1/(s-q_{l-1})}, (lc_1)^{1/s}   \right\}
	\end{align*}
	\end{lemma}
	Therefore, we can choose an $\epsilon_2$ which satisfies
	\begin{align*}
		\epsilon_2 &\leq \max\left\{  \left(\frac{2 c_1}{c_2}\right)^{1/2} , \left(\frac{2 c_0}{c_2}\right)^{1/(2+s)}  \right\}
				= \max \left\{  \left(\frac{64 \ln(\frac{4e}{\delta}) G_2^2  \sum_{i=1}^{n-1} \frac{1}{m_i}   }{(n-1)^2}\right)^{1/2} , \left(\frac{64 c_s C_{\XX} G_2^2  \sum_{i=1}^{n-1} \frac{1}{m_i}  }{(n-1)^2}\right)^{1/(2+s)}   \right \}\\
		& \leq O\left(   \sqrt{ \left( C_{\XX} + \ln\left(\frac{1}{\delta}\right) \right) \frac{G_2^2}{n^2}  \sum_{i=1}^{n} \frac{1}{m_i}  }  \right)
	\end{align*}

	\paragraph{\textit{Error Bound}}
	Suppose $m_n = m_0 n^{r}$, 
	for $r \geq 0$. Now we combine the two bounds above: fix $n\geq 2$, with probability at least $1-\delta$, 
	\begin{align*}
	 \norm{ \xi_n(x_n)  }_*  + \norm{\bar{\xi}_{n-1}(x_n)}_* 
	&\leq O\left( 
	\sqrt{ \frac{ G_2^2}{m_0 n^{r}} \ln\left( \frac{1}{\delta}  \right)  } 
	+   \sqrt{ \left( C_{\XX} + \ln\left(\frac{1}{\delta}\right) \right) \frac{G_2^2}{ m_0 n^2}  \sum_{i=1}^{n} \frac{1}{i^r}  } 
	\right)
	\end{align*}
	Due to the nature of harmonic series, we consider two scenarios.
	\begin{enumerate}
		\item If $r \in [0,1]$, then the bound can be simplified as
		\begin{align*}
		&O\left( 
		\sqrt{ \frac{ G_2^2}{m_0 n^{r}} \ln\left( \frac{1}{\delta}  \right)  } 
		+   \sqrt{ \left( C_{\XX} +  \ln\left(\frac{1}{\delta}\right)  \right) \frac{G_2^2}{ m_0 n^2}  \sum_{i=1}^{n} \frac{1}{i^r}  } 
		\right)  \\
		&=  O\left( 
		\sqrt{ \frac{ G_2^2}{m_0 n^{r}} \ln\left( \frac{1}{\delta}  \right)  } 
		+   \sqrt{ \left( C_{\XX} +  \ln\left(\frac{1}{\delta}\right) \right) \frac{G_2^2 n^{1-r }}{ m_0 n^2}    } 
		\right)
		=  O\left( 
		G_2 \sqrt{  \frac{ \ln(1/\delta) }{m_0 n^{r}}  } 
		+  G_2 \sqrt{  \frac{ C_{\XX}  }{ m_0 n^{1+r}}    } 
		\right)		
		\end{align*}

		\item If $r > 1$, then the bound can be simplified as 
		\begin{align*}
		&O\left( 
		\sqrt{ \frac{ G_2^2}{m_0 n^{r}} \ln\left( \frac{1}{\delta}  \right)  } 
		+   \sqrt{ \left( C_{\XX} + \ln\left(\frac{1}{\delta}\right) \right) \frac{G_2^2}{ m_0 n^2}  \sum_{i=1}^{n} \frac{1}{i^r}  } 
		\right)  \\
		&=  O\left( 
		\sqrt{ \frac{ G_2^2}{m_0 n^{r}} \ln\left( \frac{1}{\delta}  \right)  } 
		+   \sqrt{ \left( C_{\XX} +   \ln\left(\frac{1}{\delta}\right)  \right) \frac{G_2^2}{ m_0 n^2}    } 
		\right)
		=  O \left( G_2  \sqrt{\frac{\ln(1/\delta)   }{m_0 n^{ \min\{r,2\}  }}}  \right) + O \left( G_2  \sqrt{\frac{ C_{\XX}  }{m_0 n^2  }}  \right)
		\end{align*}
		
	\end{enumerate}

 	Therefore, we conclude for $r \geq 0 $, 
	\begin{align*}
	\norm{ \xi_n(x_n)  }_*  + \norm{\bar{\xi}_{n-1}(x_n)}_* = O\left( 
	\sqrt{  \frac{ G_2^2 \ln(1/\delta) }{m_0 n^{ \min\{r,2\} }}  } 
	+   \sqrt{   \frac{ G_2^2 C_{\XX} }{ m_0 n^{ 1+\min\{r,1\}   }}    } 
	\right)	
	\end{align*}
  Combining this inequality with Lemma~\ref{lm:bound on consecutive iterates (error)}  gives the final statement. 
\end{proof}

\subsubsection{Bound on $S_n$} \label{sec:bound on S_n (stochastic)}

Now we use Lemma~\ref{lm:bound on consecutive iterates (stochastic)} to refine Proposition~\ref{pr:bound on S_n} for stochastic problems. 
\begin{proposition} \label{pr:bound on S_n (stochastic)}
	Under the assumptions Proposition~\ref{pr:bound on S_n}, suppose $m_n = m_0 n^{r}$. For a fixed $n\geq 2$, the following holds with probability at least $1-\delta$.
	\begin{align*}
		S_n \leq \tilde{O}\left( 	\frac{G_2}{\alpha \sqrt{m_0}} \left(  \frac{ \sqrt{\ln(1/\delta)} }{ n^{\min\{r/2,1,1-\theta\}} } + \frac{ \sqrt{C_{\XX}}  }{ n^{\min\{(1+r)/2,1,1-\theta\}} } \right) \right)  
	\end{align*}
	
\end{proposition}
\begin{proof}
	The proof is similar to that of Proposition~\ref{pr:bound on S_n}, but we use the results from  Lemma~\ref{lm:bound on consecutive iterates (stochastic)}. Note Lemma~\ref{lm:bound on consecutive iterates (stochastic)} holds for a particular $n$. Here need the bound to apply for all $n=1\dots N$ so we can apply the bound for each $S_n$. This will add an additional $\sqrt{\ln N}$ factor to the bounds in  Lemma~\ref{lm:bound on consecutive iterates (stochastic)}.
	
	First, we recall that
	\begin{align*}
	S_{n+1} 
	&\leq \left( 1- \frac{1}{n} \right) S_n + \norm{x_{n+1}-x_n} 
	\end{align*}
     By Lemma~\ref{lm:bound on consecutive iterates (stochastic)}, let  $c_1 = 
     \frac{G_2 \sqrt{\ln(1/\delta)} }{n \alpha \sqrt{m_0}}  $ and $c_2 = \frac{G_2 \sqrt{ C_{\XX} } }{n \alpha \sqrt{m_0}}  $, and it holds that
     \begin{align*}
     \norm{ x_{n+1} -x_n} \leq   \frac{\theta S_n}{n} +   O\left( 
     \frac{G_2}{n \alpha \sqrt{m_0}} \left( \sqrt{  \frac{  \ln(1/\delta) }{ n^{ \min\{r,2\} }}  } + \sqrt{   \frac{  C_{\XX} }{  n^{ 1+\min\{r,1\}   }}    } \right) 
     \right)	
     = \frac{\theta S_n}{n} + O(\frac{c_1}{n^{1+\min\{r,2\}/2}} + \frac{c_2}{n^{3/2+ \min\{r,1\}/2}})
     \end{align*}    
	which implies
	\begin{align*}
	S_{n+1} 
	\leq \left( 1- \frac{1}{n} \right) S_n + \norm{x_{n+1}-x_n} \leq  \left(1 - \frac{1-\theta}{n}\right) S_n + O(\frac{c_1}{n^{1+\min\{r,2\}/2}} + \frac{c_2}{n^{3/2+ \min\{r,1\}/2}}).
	\end{align*}
	Recall 
	\DTdynamics*	
	From Proposition~\ref{pr:bound on S_n}, we know the unperturbed dynamics is bounded by  $ e^{1-\theta} n^{\theta -1}  S_2$ (and can be shown in $\Theta(n^{\theta -1})$ as in the proof of Theorem~\ref{th:lower bound}). To consider the effect of the perturbations, due to linearity we can treat each perturbation separately and combine the results by superposition.	
	Suppose a particular perturbation is of the form $O(\frac{C_2}{n^{1+s}})$ for some $C_2$ and $s>0$. By Lemma~\ref{lm:convergence DT dynamical system}, suppose  $ \theta + s < 1$,
	\begin{align*}
	S_n &\leq O(n^{\theta-1}) + O\left( n^{\theta - 1} \sum_{k=1}^n  k^{1 - \theta} \frac{C_2}{k^{1+s}}\right)
	\leq O(n^{\theta-1}) + O\left( C_2 n^{\theta - 1} n^{1-s-\theta}  \right)  = O(n^{\theta -1}) + O\left( C_2 n^{-s}  \right) 
	\end{align*}
	For $\theta - s = 1$, $S_n \leq O(n^{\theta-1}) + O(C_2 n^{\theta-1} \ln(n))$; for $ \theta + s > 1$, $S_n \leq O(n^{\theta-1}) + O(C_2 n^{\theta-1} )$. 
	Therefore, we can conclude $S_n \leq C_1 n^{\theta-1} + \tilde{O} (C_2 n^{-\min\{s, 1-\theta \} }) $, where the constant $C_1 = e^{1-\theta} S_2$. Finally, using $S_2 \leq \frac{G_2}{\alpha}$ and setting $C_2 $ as $c_1$ or $c_2$ gives the final result
	\begin{align*}
	S_n \leq \tilde{O}\left( 	\frac{G_2}{\alpha \sqrt{m_0}} \left(  \frac{ \sqrt{\ln(1/\delta)} }{ n^{\min\{r/2,1,1-\theta\}} } + \frac{ \sqrt{C_{\XX}}  }{ n^{\min\{(1+r)/2,1,1-\theta\}} } \right) \right) 
	\end{align*}
\end{proof}

\subsubsection{Performance Guarantee}
Given Proposition~\ref{pr:bound on S_n (stochastic)}, now we can prove the performance of the last iterate. 
\thmStochastic*
\begin{proof}
	The proof is similar to the proof of Theorem~\ref{th:convergence of the last iterate}.
	Let $x_n^* \coloneqq \argmin_{x \in \XX} f_n(x)$. Then 
	\begin{align*}
	f_n(x_n) - \min_{x \in \XX} f_n(x) 
	&\leq \lr{\nabla f_n(x_n)  }{ x_n - x_n^*} - \frac{\alpha}{2} \norm{x_n - x_n^*}^2 \\
	&\leq \lr{\nabla f_n(x_n) - \nabla \bar{f}_{n-1}(x_n)  }{ x_n - x_n^*} 
	+ \lr{\nabla \bar{f}_{n-1}(x_n) - \nabla \bar{g}_{n-1}(x_n)  }{ x_n - x_n^*} - \frac{\alpha}{2} \norm{x_n - x_n^*}^2 \\
	&\leq \norm{\nabla f_n(x_n) - \nabla \bar{f}_{n-1}(x_n)}_*\norm{ x_n - x_n^*} + \norm{\nabla \bar{g}_{n-1}(x_n) - \nabla \bar{f}_{n-1}(x_n) }_*\norm{{ x_n - x_n^*}}- \frac{\alpha}{2} \norm{x_n - x_n^*}^2 \\
	&\leq \frac{(\norm{\nabla f_n(x_n) - \nabla \bar{f}_{n-1}(x_n)}_* +  \norm{\nabla \bar{g}_{n-1}(x_n) - \nabla \bar{f}_{n-1}(x_n) }_*)^2}{2\alpha} \\
	&\leq \frac{\norm{\nabla f_n(x_n) - \nabla \bar{f}_{n-1}(x_n)}_*^2 +  \norm{\nabla \bar{g}_{n-1}(x_n) - \nabla \bar{f}_{n-1}(x_n) }_*^2}{\alpha} 
	\end{align*}
	where the second inequality is due to $x_n = \argmin_{x\in \XX} \bar{g}_{n-1}(x)$. 
	To bound the first term, recall the fact that $\norm{\nabla f_n(x_n) - \nabla \bar{f}_{n-1}(x_n)}_* < \beta S_n$ and recall  by  Proposition~\ref{pr:bound on S_n (stochastic)} that 
	\begin{align*}
	S_n \leq \tilde{O}\left( 	\frac{G_2}{\alpha \sqrt{m_0}} \left(  \frac{ \sqrt{\ln(1/\delta)} }{ n^{\min\{r/2,1,1-\theta\}} } + \frac{ \sqrt{C_{\XX}}  }{ n^{\min\{(1+r)/2,1,1-\theta\}} } \right) \right) 
	\end{align*}	
	For the second term, we use the proof in Lemma~\ref{lm:bound on consecutive iterates (stochastic)} with an additional $\ln(N)$ factor, i.e.
	\begin{align*}
	\norm{\nabla \bar{g}_{n-1}(x_n) - \nabla \bar{f}_{n-1}(x_n) }_* 
	= \tilde{O} \left(  \frac{G_2}{\sqrt{m_0}}  \sqrt{\frac{\ln(1/\delta) + C_{\XX }}{ n^{1+\min\{r,1\}  }}} \right)
	\end{align*}
	Let $c = \frac{\alpha m_0}{G_2^2}$. 
	Therefore, combining all the results, we have the following with probability at least $1-\delta$: 
	\begin{align*}
	f_n(x_n) - \min_{x \in \XX} f_n(x) 	&\leq  \frac{\norm{\nabla f_n(x_n) - \nabla \bar{f}_{n-1}(x_n)}_*^2 +  \norm{\nabla \bar{g}_{n-1}(x_n) - \nabla \bar{f}_{n-1}(x_n) }_*^2}{\alpha} \\
	&\leq \frac{\beta^2 S_n^2  }{\alpha} + \frac{\norm{\nabla \bar{g}_{n-1}(x_n) - \nabla \bar{f}_{n-1}(x_n) }_*^2}{\alpha}\\
	&\leq  \tilde{O}\left(  \frac{\theta^2 G_2^2}{\alpha m_0}  \frac{\ln(1/\delta)}{  n^{2\min\{r/2,1,1-\theta\} } }  \right)  +\tilde{O}\left(  \frac{\theta^2 G_2^2}{\alpha m_0}  \frac{C_{\XX}}{  n^{2\min\{(r+1)/2,1,1-\theta\} } }  \right)
	+  \tilde{O} \left( \frac{G_2^2}{\alpha m_0}  \frac{\ln(1/\delta) + C_{\XX }}{ n^{1+\min\{r,1\}}} \right) \\
	&= \tilde{O}\left(  \frac{\theta^2}{c}  \frac{\ln(1/\delta)}{  n^{2\min\{r/2,1,1-\theta\} } }  \right)  +\tilde{O}\left(  \frac{\theta^2 }{c}  \frac{C_{\XX}}{  n^{2\min\{(r+1)/2,1,1-\theta\} } }  \right)
	+  \tilde{O} \left( \frac{1}{c}  \frac{\ln(1/\delta) + C_{\XX }}{ n^{1+\min\{r,1\}}} \right) \\
	&\leq \tilde{O}\left(  \frac{\theta^2}{c}  \frac{\ln(1/\delta) + C_{\XX}}{  n^{2\min\{r/2,1,1-\theta\} } }  \right) 
	+  \tilde{O} \left(   \frac{\ln(1/\delta) + C_{\XX }}{ c n^{1+\min\{r,1\}}} \right)
	\end{align*}
	Note the last inequality is unnecessary and is used to simplify the result. It can be seen that the upper bound originally has a weaker dependency on $C_{\XX}$.

\end{proof}

\section{\aggrevate with Function Approximations} \label{sec:function approximation}

Here we give a sketch of applying the techniques used in Theorem~\ref{th:convergence of the last iterate (stochastic)} to problems where a function approximator is used to learn $f(\cdot;s)$, as in the case considered by \citet{ross2011reduction} for learning the Q-function. 

We consider a meta learning scenario where a linear function approximator $\hat{f}(x,s) = \phi(x,s)^T w$ is used to approximate $f(x;s)$. We assume $\phi(x,s)^T w$ satisfies Assumption~\ref{as:structured function} 
and Assumption~\ref{as:uniform continous gradient2} with some appropriate constants.

Now we analyze the case where $\sum_{i=1}^{m_n}\hat{f}(\cdot, s_{n,i})$ is used as the per-round cost in \aggrevate. Specifically, in the $n$th iteration of \aggrevate, $m_n$ samples$\{ f(x_n; s_{n,k})  \}_{k=1}^{m_n}$ are first collected, and then $w_n$ is updated by 
\begin{align}
	w_n = \argmin_{w \in \WW} \sum_{i=1}^{n} \sum_{j=1}^{m_i} \left( f (x_i; s_{i,j} ) - \phi(x_i,s_{i,j})^T w \right)^2  \label{eq:update w_n}
\end{align}
where $\WW$ is the domain of $w$. Given the new $w_n$, the policy is updated by 
\begin{align}
	x_{n+1} = \argmin_{x\in\XX} \sum_{i=1}^{n} \sum_{j=1}^{m_i} \phi(x_i,s_{i,j})^T w_n  \label{eq:update x_{n+1} given w_n}
\end{align}

To prove the performance, we focus on the inequality used in the proof of performance in Theorem~\ref{th:convergence of the last iterate (stochastic)}.
	\begin{align*}
	f_n(x_n) - \min_{x \in \XX} f_n(x) 
	&\leq \lr{\nabla f_n(x_n)  }{ x_n - x_n^*} - \frac{\alpha}{2} \norm{x_n - x_n^*}^2 	
	\end{align*}
And we  expand the inner product term: 	
	\begin{align*}
	\lr{\nabla f_n(x_n)  }{ x_n - x_n^*} = 
	\lr{\nabla \bar{g}_{n;w_{n-1}} (x_n) }{ x_n - x_n^*} + \lr{\nabla \bar{g}_{n;w_{n}} - \nabla \bar{g}_{n;w_{n-1}} }{ x_n - x_n^*}  + \lr{\nabla f_n(x_n) - \nabla \bar{g}_{n;w_{n}} }{ x_n - x_n^*}
	\end{align*}
where $\bar{g}_{n;w_{n}}$ is the finite-sample approximation using $w_{n}$ .  By~\eqref{eq:update x_{n+1} given w_n}, $x_{n} = \argmin_{x\in\XX} \bar{g}_{n;w_{n-1}}(x)$, and therefore
	\begin{align*}
		\lr{\nabla f_n(x_n)  }{ x_n - x_n^*} \leq \lr{\nabla \bar{g}_{n;w_{n}} - \nabla \bar{g}_{n;w_{n-1}} }{ x_n - x_n^*}  + \lr{\nabla f_n(x_n) - \nabla \bar{g}_{n;w_{n}} }{ x_n - x_n^*}
	\end{align*}
	In the first term,  $\norm{\nabla \bar{g}_{n;w_{n}} - \nabla \bar{g}_{n;w_{n-1}}}_{*} \leq O(\norm{ w_n - w_{n-1}})$. As $w_n$ is updated by another value aggregation algorithm, this term can be further bounded similarly as in Lemma~\ref{lm:bound on consecutive iterates}, by assuming a similar condition like Assumption~\ref{as:uniform continous gradient2} but on the change of the gradient in the objective function in~\eqref{eq:update w_n}. In the second term, $\norm{\nabla f_n(x_n) - \nabla \bar{g}_{n;w_{n}}}_*$ can be bounded by the uniform bound of vector-valued martingale in Lemma~\ref{lm:mean vector-valued function martingale (uniform)}. Given these two bounds, it follows that
	\begin{align*}
	f_n(x_n) - \min_{x \in \XX} f_n(x) 
	&\leq \frac{\norm{\nabla \bar{g}_{n;w_{n}} - \nabla \bar{g}_{n;w_{n-1}}}_{*} ^2 + \norm{\nabla f_n(x_n) - \nabla \bar{g}_{n;w_{n}}}_*^2 }{\alpha} 
	\end{align*}
	Compared with Theorem~\ref{th:convergence of the last iterate (stochastic)}, since here additional Lipschitz constant is introduced to bound the change $\norm{\nabla \bar{g}_{n;w_{n}} - \nabla \bar{g}_{n;w_{n-1}}}_{*}$, one can expect that the stability constant $\theta$ for this meta-learning problem will increase.

\section{Weighted Regularization} \label{sec:general weighted regularization}

Here we discuss the case where $R(x) = F(\pi^*, x)$ regardless the condition $R(x) \geq 0$. 
\begin{corollary} \label{cr:performance with weighted regularization extra} 
	Let $\tilde{F}(x,x) = F(x,x) + \lambda F(\pi^*,x)$.
	Suppose $\forall x\in \XX$, $\min_{x\in \XX} \tilde{F}(x,x) \leq (1+\lambda)  \tilde{\epsilon}_{\Pi,\pi^*}.$ Define 
	$\Delta_N =  (1+\lambda)  \frac{( \tilde{\theta} e^{1-\tilde{\theta}}  G_2 )^2}{2\alpha}  N^{2(\tilde{\theta} -1)}.$
	Running \aggrevate with $\tilde{F}$ in~\eqref{eq:additive regularization} as the per-round cost has performance satisfies: for all $N >0$,
	\begin{align*}
	F(x_N, x_N) &\leq (1+\lambda) \tilde{\epsilon}_{\Pi,\pi^*}  - \lambda F(x^*, x_N) + \Delta_N\\
	&\leq \Delta_N +  \tilde{\epsilon}_{\Pi, \pi^*}  + \lambda G_2 \left( \frac{2 \lambda G_2}{\alpha} + \sqrt{\frac{2\Delta_N}{\alpha}}\right)
	\end{align*}
\end{corollary}
\begin{proof}
	The first inequality can be seen by the definition 	
	$	F(x_N, x_N) = \tilde{F}(x_N, x_N) - \lambda F(x^*, x_N) $ and then by applying Theorem~\ref{th:convergence of the last iterate} to $\tilde{F}(x_N, x_N)$.
	
	The second inequality shows that  $-F(x^*,x_N)$ cannot be too large. Let $f_*(x) = F(x^*,x)$ and $x_N^* = \argmin_{x\in \XX} f_N(x)$.  Then 
	\begin{align*}
	f_N(x_N) &= f_N(x_N) + \lambda f_*(x_N) - \lambda f_*(x_N)\\
	&\leq \Delta_N  - \lambda f_*(x_N)  + \min_{x\in \XX} f_N(x) + \lambda f_*(x) \\
	&\leq \Delta_N + f_N(x_N^*) + \lambda ( f_*(x_N^*) - f_*(x_N  ))\\
	&\leq \Delta_N +  f_N(x_N^*)  + \lambda G_2 \norm{x_N^* - x_N}
	\end{align*}
	where the first inequality is due to Theorem~\ref{th:convergence of the last iterate} and the third inequality is due to $f_*$ is $G_2$-Lipschitz continuous. Further, since $f_N$ is $\alpha$-strongly convex, 
	\begin{align*}
	\frac{\alpha}{2}\norm{x_N^* - x_N}^2 &\leq f_N(x_N) - f_N(x_N^*) \\
	&\leq \Delta_N  + \lambda G_2 \norm{x_N^* - x_N}
	\end{align*}
	which implies
	\begin{align*}
	\norm{x_N^* - x_N} &\leq \frac{\lambda G_2 + \sqrt{\lambda^2 G^2 + 2 \alpha \Delta_N}}{\alpha} \\
	&\leq \frac{2 \lambda G_2 + \sqrt{ 2 \alpha \Delta_N}}{\alpha} 
	\end{align*}
	Therefore, 
	\begin{align*}
	f_N(x_N) &\leq \Delta_N +  f_N(x_N^*)  + \lambda G_2 \norm{x_N^* - x_N} \\
	&\leq \Delta_N +  \tilde{\epsilon}_{\Pi, \pi^*}  + \lambda G_2 \left( \frac{2 \lambda G_2}{\alpha} + \sqrt{\frac{2\Delta_N}{\alpha}}\right)
	\end{align*}
\end{proof}

Corollary~\ref{cr:performance with weighted regularization extra} indicates that when $\pi^*$ is better than all policies under the distribution of $\pi^*$ (i.e. $F(x^*,x) \geq 0, \forall x \in \XX$), then using \aggrevate with the weighted problem such that $\tilde{\theta}<1$ generates a convergent sequence and then the performance on the last iterate is bounded by $(1+\lambda) \tilde{\epsilon}_{\Pi,\pi^*} + \Delta_N$. That is, it only introduces a multiplicative constant on $ \tilde{\epsilon}_{\Pi,\pi^*}$. Therefore, the bias due to regularization can be ignored by choosing a larger policy class. This suggests for applications like \dagger introducing additional weighted cost $\lambda F(x^*,x)$ (i.e. demonstration samples collected under the expert policy's distribution) does not hurt. 

However, in generally, $F(x^*, x_N)$ can be negative, when there is a better policy in $\Pi$ than $\pi^*$ in sense of the state distribution $d_{\pi^*}(s)$ generated by the expert policy $\pi^*$. 
Corollary~\ref{cr:performance with weighted regularization extra} also shows this additional bias introduced by \aggrevate is bounded at most $O(\frac{\lambda^2 G_2^2}{\alpha})$.

\end{document}